\documentclass{style} 

\usepackage[T1]{fontenc}
\usepackage{lmodern}
\usepackage{microtype}
\usepackage[utf8]{inputenc}

\usepackage{setspace}
\usepackage{booktabs}
\usepackage{multirow}
\usepackage{makecell}
\usepackage{tabularx}
\usepackage{titletoc}

\usepackage{titlesec}

\usepackage{amsmath,amsfonts,bm}

\def\eqref#1{equation~\ref{#1}}

\def\1{\bm{1}}

\DeclareMathAlphabet{\mathsfit}{\encodingdefault}{\sfdefault}{m}{sl}
\SetMathAlphabet{\mathsfit}{bold}{\encodingdefault}{\sfdefault}{bx}{n}

\usepackage{graphicx} 
\usepackage{booktabs}
\usepackage[most]{tcolorbox}
\usepackage{enumitem}
\usepackage{listings}
\usepackage{multirow}
\usepackage{wrapfig}
\usepackage{xcolor}
\usepackage[authoryear,round]{natbib}
\usepackage{amsthm}
\usepackage{fvextra}
\usepackage[most]{tcolorbox}
\tcbuselibrary{skins,breakable}
\usepackage[scaled=1.10]{inconsolata}
\usepackage{subcaption}
\fvset{fontsize=\scriptsize, breaklines=true, breakanywhere=true,
       breaksymbolleft={}}

\definecolor{warmbg}{HTML}{FFF8F2}
\definecolor{warmtitle}{HTML}{E8D2BF}
\definecolor{warmframe}{HTML}{CFA98B}

\newtcolorbox{rollout}[1]{%
  enhanced jigsaw,
  breakable,
  colback=warmbg,
  colframe=warmframe,
  boxrule=0.4pt,
  arc=1pt,
  left=4pt,
  right=4pt,
  top=3pt,
  bottom=3pt,
  title={\footnotesize\bfseries #1},
  coltitle=black,
  colbacktitle=warmtitle,
}
\theoremstyle{plain}
\newtheorem{proposition}{Proposition}

\definecolor{tokBlue}{HTML}{2F5D8A}     
\definecolor{tokOrange}{HTML}{C2702A}   
\definecolor{tokGreen}{HTML}{2E7D5B}    
\definecolor{tokPurple}{HTML}{7B4F9D}   
\definecolor{tokBrown}{HTML}{8C6D3F}    
\definecolor{tokRed}{HTML}{B03A48}      
\definecolor{tokOlive}{HTML}{6E8B3D}    
\definecolor{stageShade}{gray}{0.95}    

\newcommand{\ckpt}[2]{\textcolor{#1}{#2}}
\newcommand{\tok}[3]{\textcolor{#1}{\texttt{#2}\,\small(#3)}}
\definecolor{linkblue}{HTML}{547A9E}

\newtcblisting{promptbox}[1]{
    enhanced,
    breakable,
    listing only,
    title={#1},
    colback=warmbg,
    colframe=warmframe,
    coltitle=black,
    colbacktitle=warmtitle,
    fonttitle=\bfseries,
    boxrule=0.6pt,
    arc=1.5mm,
    left=2mm,
    right=2mm,
    top=1mm,
    bottom=1mm,
    listing options={
        basicstyle=\ttfamily\small,
        breaklines=true,
        columns=fullflexible,
        keepspaces=true,
        showstringspaces=false
    }
}
\usepackage{url}
\usepackage{hyperref}
\usepackage[nameinlink,capitalise]{cleveref}
\title{When EOS Tokens Disagree: Understanding Length Inflation in On-Policy Distillation}

\author[1,*]{Yuxiao Yang}
\author[2,*]{Tianrun Yu}
\author[1,*]{Shangzhe Li}
\author[2,*]{Kaixiang Zhao}
\author[3]{Xuchao Zhang}
\author[3]{Chetan Bansal}
\author[1]{Huaxiu Yao}
\author[2,\dagger]{Taylor W. Killian}
\author[1,\ddagger]{Weitong Zhang}

\contribution[*]{Equal contribution.}
\contribution[\dagger]{\texttt{tkillian@cs.byu.edu}}
\contribution[\ddagger]{\texttt{weitongz@unc.edu}}

\affiliation[1]{University of North Carolina at Chapel Hill}
\affiliation[2]{Brigham Young University}
\affiliation[3]{Microsoft}

\begin{document}

\abstract{
\vspace{-.7em}
We study length inflation in on-policy distillation (OPD), where student responses can become excessively long and even exhaust the generation budget. We identify \emph{termination-token mismatch} between base students and post-trained teachers as an important source of this behavior. Across Qwen3, Llama, and Gemma, the two models can place their stopping probability on different EOS tokens, even when their declared stopping sets are identical. This mismatch can suppress the student's preferred termination action without reliably transferring the teacher-preferred alternative. We show that aligning the decoding stopping set alone is insufficient, while treating functionally equivalent EOS tokens as a shared semantic stopping action substantially mitigates mismatch-induced length inflation across all three model families. To further understand how termination behavior evolves over training, we study OPD across different K2-Horizon training stages. This stage-wise analysis shows that termination preferences can shift substantially during training, while also revealing a distinct length inflation late in the OPD run that persists beyond termination alignment. Together, these results identify termination mismatch as an important, but not exhaustive, source of OPD length dynamics. We release an implementation incorporating the proposed termination-handling corrections.

\code{https://uncsciml.github.io/opd-eos-website} 
}

\maketitle

\section{Introduction}

On-policy distillation (OPD) trains a student using its own rollouts with dense token-level supervision from a teacher model \citep{lu2025onpolicydistillation}. It has recently attracted growing attention as a practical framework for transferring strong post-trained
behavior to smaller or less capable models
\citep{EOPD,EXOPD,li2026rethinking}.
Despite its effectiveness, several recent studies have reported undesired length inflation in student responses, often accompanied by truncation, repetition, or other generation pathologies \citep{demisify}.
We observe the same phenomenon across several model families, using pretrained base students and post-trained teachers in single-turn mathematical reasoning. Student responses become progressively longer during training, and in severe cases a large fraction of rollouts
exhaust the generation budget. We refer to this behavior as \emph{length inflation}. \Cref{fig:length_inflation_overview} illustrates a
Qwen3 example in which the student reaches the correct answer early but continues generating repetitive or redundant suffixes instead of terminating. Such increases in response length need not reflect more
extensive reasoning; they can instead arise from a failure to stop. Existing studies have proposed different explanations and mitigation strategies for this phenomenon, ranging from objective-level effects of reverse-KL distillation to rollout degradation and other training instabilities \citep{fu2026revisiting,simpleOPD,demisify,
wang2026demystifyingonpolicydistillationroles,zhao2026decoupling}.

In this work, however, we find that a direct cause of this inflation in the configurations studied is \emph{termination-token mismatch} between base
students and post-trained teachers. We observe this mismatch across Qwen3, Llama, and Gemma. The key issue is that the student and teacher may represent
the same semantic decision to terminate using different EOS tokens. This mismatch can arise not only when their declared decoding stopping sets differ, but also when the two models assign their learned stopping probability to different tokens within the same declared set. Since the teacher only evaluates the tokens generated by students in OPD, a student-generated EOS token receives a negative distillation signal if teacher model is actively using an alternative EOS token. The student's existing termination action can therefore be repeatedly suppressed without reliably transferring the teacher's termination token with minor ($\sim10^{-11}$ in Qwen3) in student model.

To address this mismatch, we first show that aligning the decoding interface alone is insufficient. We then compare three corrections that align termination in the distillation signal and find that treating functionally
equivalent EOS tokens as a shared semantic stopping action substantially mitigates mismatch-induced length inflation across Qwen3, Llama, and Gemma in single-turn mathematical reasoning.

\begin{figure}[t]
    \centering

    \begin{tcolorbox}[
        enhanced,
        colback=gray!3,
        colframe=gray!12,
        boxrule=0.3pt,
        arc=1.5mm,
        left=2mm,
        right=2mm,
        top=2mm,
        bottom=2mm
    ]

    \begin{subfigure}{\linewidth}
        \centering
        \includegraphics[width=\linewidth]{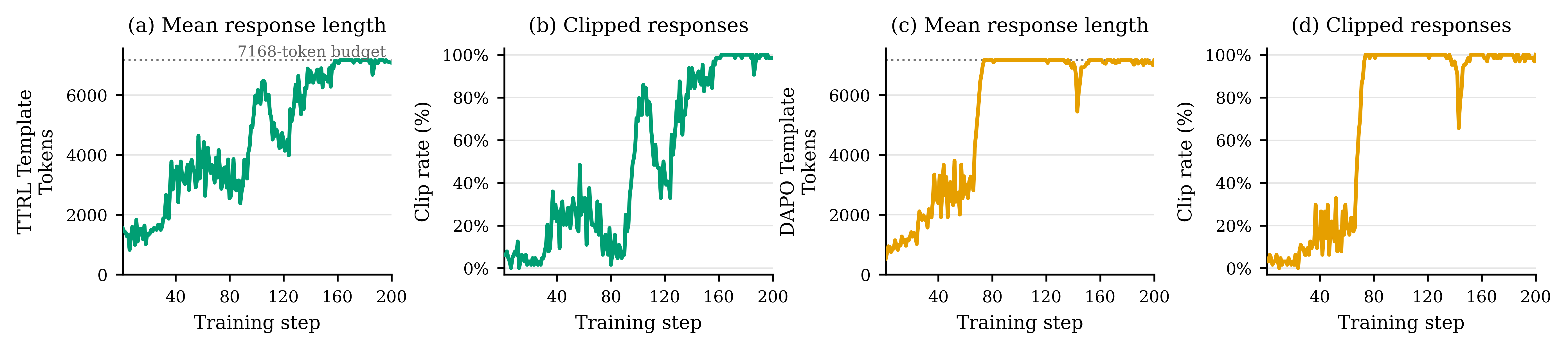}
        \caption{Evolution of response length and length-clipping rate during training.}
        \label{fig:baseline_opd}
    \end{subfigure}

    \vspace{0.5em}

    \begin{subfigure}{\linewidth}
        \centering
        \includegraphics[width=\linewidth]{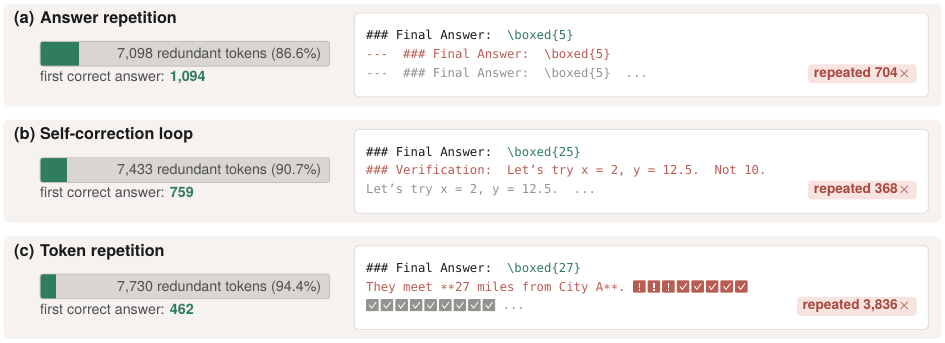}
        \caption{Qualitative examples of failure to terminate.}
        \label{fig:demo}
    \end{subfigure}

    \end{tcolorbox}

    \caption{
    \textbf{Length inflation under vanilla OPD, illustrated with Qwen3.}
    (a) When distilling a post-trained Qwen3 teacher into a Qwen3 base
    student, response length increase and hits the generation budget during training under both TTRL and DAPO
    prompts.
    (b) Inflated responses often contain the correct answer early but
    fail to terminate, consuming much of the remaining generation budget
    with repetitive or redundant continuations.
    }
    \label{fig:length_inflation_overview}
\end{figure}

To further understand how this termination mismatch evolves over training, we leverage the intermediate K2-Horizon~\cite{} checkpoints spanning pretraining, midtraining, supervised fine-tuning, and final post-training. This stage-wise analysis shows how learned termination preferences shift across training stages, but also reveals a distinct inflation late in the OPD run:
even after the teacher-preferred termination form has been transferred, response length can increase again and termination probability can collapse. This behavior is not explained by the EOS mismatch alone and points to
additional mechanisms governing OPD length dynamics. We also observe that prompt-template choices can further modulate response length and measured performance.

In short, beyond the mechanisms considered in prior work, we identify termination-token mismatch between base and post-trained models as a major and readily diagnosable source of length inflation. Our contributions are summarized as follows:

\begin{itemize}[
    labelindent=0pt,
    leftmargin=*,
    itemsep=-0.2em,
]
\item We identify termination-token mismatch as an important mechanism of length inflation in sampled-token OPD. Across Qwen3, Llama, and Gemma, we show that student and teacher checkpoints can represent the same semantic stopping decision through different learned token preferences, even when their declared EOS sets are identical.

\item We systematically compare four termination-handling strategies on Qwen3 and show that aligning the decoding stopping set alone is insufficient. Corrections that align termination semantics in the distillation signal substantially mitigate mismatch-induced inflation, and we validate semantic EOS aggregation across Qwen3, Llama, and Gemma.

\item We extend the analysis across post-training stages using K2-Horizon, showing how learned termination preferences evolve during training and
revealing a distinct length inflation late in the OPD run from the pretraining checkpoint (Phase~3 in \Cref{fig:k2_stage_alignment}), which remains after termination alignment. We also document sensitivity to prompt and evaluation protocols, highlighting these choices as an additional source of variation in reported OPD behavior.

\item We release an implementation that includes the termination-handling corrections and the evaluation protocols.

\end{itemize}

\section{Experimental Setup}
\label{sec:setup}
\label{sec:preliminaries}

\noindent\textbf{Sampled-token OPD and notation.}
On-policy distillation trains a student $\pi_\theta$ on its own rollouts
with a fixed teacher $\pi^E$. At a prefix $s_t=(x,y_{<t})$, the sampled
token $y_t$ receives the signal
\[
A_t=\log\pi^E(y_t\mid s_t)-\log\pi_\theta(y_t\mid s_t),
\qquad
g_t=A_t\nabla_\theta\log\pi_\theta(y_t\mid s_t).
\]
We treat $A_t$ as detached when we compute the score-function update, and we
average the contributions of all valid generated tokens without reward-to-go.
We refer to this formulation as sampled-token OPD, in which the teacher
scores only the token that the student has sampled. We write
$\mathcal{E}_{\mathrm{EOS}}$ for the set of termination tokens that are
equivalent under the rollout protocol of a given model pair, and
$e_\star\in\mathcal{E}_{\mathrm{EOS}}$ for the termination token that the
base student supports natively. The corrections in
\Cref{sec:length_inflation} all act on these two objects. Our experiments
build on the public OPD implementation of \citet{li2026rethinking}.

\newcommand{\familyfour}[1]{%
    \multirow[c]{4}{*}[-2.0em]{#1}%
}
\newcommand{\familythree}[1]{%
    \multirow[c]{3}{*}[-2.8em]{#1}%
}
\newcommand{\stagetwo}[1]{%
    \multirow[c]{2}{*}[-0.7em]{#1}%
}
\begin{table}[t]
    \centering
    \caption{
    \textbf{EOS tokens of model studied.}
    We report the \texttt{eos\_token\_id} specified in each model's
    \texttt{generation\_config.json}. Models sharing the same declared
    EOS token ID set are shown in the same color; token IDs are given in
    parentheses. For \texttt{IFM/K2-Horizon-7B}, revision names identify
    the training-stage models used in our experiments.
    Its EOS token strings were renamed after pretraining, while their
    IDs and the declared stopping set remained unchanged.
    }
    \label{tab:eos_configs}

    \footnotesize
    \setlength{\tabcolsep}{3.5pt}
    \renewcommand{\arraystretch}{1.08}

    \begin{tabularx}{\linewidth}{
        @{}
        >{\raggedright\arraybackslash}m{0.075\linewidth}
        >{\raggedright\arraybackslash}m{0.105\linewidth}
        >{\raggedright\arraybackslash}p{0.135\linewidth}
        >{\raggedright\arraybackslash}X
        >{\raggedright\arraybackslash}p{0.285\linewidth}
        @{}
    }
        \toprule
        \textbf{Family} &
        \textbf{Stage} &
        \textbf{Version} &
        \textbf{Checkpoint(s)} &
        \textbf{EOS token(s)} \\
        \midrule

        \familyfour{Qwen}
        & \stagetwo{Base}
        & Qwen2.5-Math
        & \makecell[l]{
            \ckpt{tokBlue}{Qwen2.5-Math-1.5B}\\[-1pt]
            \ckpt{tokBlue}{Qwen2.5-Math-7B}
          }
        & \multirow[c]{2}{*}{
            \tok{tokBlue}{<|endoftext|>}{151643}
          } \\

        &
        & Qwen3
        & \makecell[l]{
            \ckpt{tokBlue}{Qwen3-1.7B-Base}\\[-1pt]
            \ckpt{tokBlue}{Qwen3-4B-Base}
          }
        & \\

        \cmidrule(lr){2-5}

        &
        \stagetwo{Post-trained}
        & Qwen2.5-Math
        & \makecell[l]{
            \ckpt{tokOrange}{Qwen2.5-Math-1.5B-Instruct}\\[-1pt]
            \ckpt{tokOrange}{Qwen2.5-Math-7B-Instruct}
          }
        & \multirow[c]{2}{*}{
            \textcolor{tokOrange}{
                \makecell[l]{
                    \texttt{<|im\_end|>, <|endoftext|>}\\[-1pt]
                    \texttt{(151645, 151643)}
                }
            }
          } \\

        &
        & Qwen3
        & \makecell[l]{
            \ckpt{tokOrange}{Qwen3-1.7B}\\[-1pt]
            \ckpt{tokOrange}{Qwen3-4B}
          }
        & \\

        \midrule

        \familyfour{Llama}
        & \stagetwo{Base}
        & Llama 3.1
        & \makecell[l]{
            \ckpt{tokGreen}{Llama-3.1-8B}\\[-1pt]
            \ckpt{tokGreen}{Llama-3.1-70B}
          }
        & \multirow[c]{2}{*}{
            \tok{tokGreen}{<|end\_of\_text|>}{128001}
          } \\

        &
        & Llama 3.2
        & \makecell[l]{
            \ckpt{tokGreen}{Llama-3.2-1B}\\[-1pt]
            \ckpt{tokGreen}{Llama-3.2-3B}
          }
        & \\

        \cmidrule(lr){2-5}

        &
        \stagetwo{Post-trained}
        & Llama 3.1
        & \makecell[l]{
            \ckpt{tokPurple}{Llama-3.1-8B-Instruct}\\[-1pt]
            \ckpt{tokPurple}{Llama-3.1-70B-Instruct}
          }
        & \multirow[c]{2}{*}{
            \textcolor{tokPurple}{
                \makecell[l]{
                    \texttt{<|end\_of\_text|>, <|eom\_id|>,}\\[-1pt]
                    \texttt{<|eot\_id|>}\\[-1pt]
                    \texttt{(128001, 128008, 128009)}
                }
            }
          } \\

        &
        & Llama 3.2
        & \makecell[l]{
            \ckpt{tokPurple}{Llama-3.2-1B-Instruct}\\[-1pt]
            \ckpt{tokPurple}{Llama-3.2-3B-Instruct}
          }
        & \\

        \midrule

        \familyfour{Gemma}
        & \stagetwo{Base}
        & Gemma 2
        & \ckpt{tokBrown}{gemma-2-2b}
        & \tok{tokBrown}{<eos>}{1} \\

        &
        & Gemma 3
        & \makecell[l]{
            \ckpt{tokOlive}{gemma-3-1b-pt}\\[-1pt]
            \ckpt{tokOlive}{gemma-3-4b-pt}
          }
        & \textcolor{tokOlive}{
            \makecell[l]{
                \texttt{<eos>, <end\_of\_turn>}\\[-1pt]
                \texttt{(1, 106)}
            }
          } \\

        \cmidrule(lr){2-5}

        &
        \stagetwo{Post-trained}
        & Gemma 2
        & \ckpt{tokRed}{gemma-2-2b-it}
        & \textcolor{tokRed}{
            \makecell[l]{
                \texttt{<eos>, <end\_of\_turn>}\\[-1pt]
                \texttt{(1, 107)}
            }
          } \\

        &
        & Gemma 3
        & \makecell[l]{
            \ckpt{tokOlive}{gemma-3-1b-it}\\[-1pt]
            \ckpt{tokOlive}{gemma-3-4b-it}
          }
        & \textcolor{tokOlive}{
            \makecell[l]{
                \texttt{<eos>, <end\_of\_turn>}\\[-1pt]
                \texttt{(1, 106)}
            }
          } \\

        \midrule

        \familythree{K2-Horizon}
        & Base
        & K2-Horizon
        & \makecell[l]{
            \ckpt{teal}{K2-Horizon-7B}\\[-1pt]
            \textcolor{teal}{\texttt{pretrain\_final}}
          }
        & \textcolor{teal}{
            \makecell[l]{
                \texttt{<|endoftext|>, <|im\_end|>}\\[-1pt]
                \texttt{(1, 250019)}
            }
          } \\

        \cmidrule(lr){2-5}

        & Mid-training
        & K2-Horizon
        & \makecell[l]{
            \ckpt{teal}{K2-Horizon-7B}\\[-1pt]
            \textcolor{teal}{\texttt{mid\_1\_final}}
          }
        & \textcolor{teal}{
            \makecell[l]{
                \texttt{<|ifm|endoftext|>,}\\[-1pt]
                \texttt{<|ifm|im\_end|>}\\[-1pt]
                \texttt{(1, 250019)}
            }
          } \\

        \cmidrule(lr){2-5}

        & Post-trained
        & K2-Horizon
        & \makecell[l]{
            \ckpt{teal}{K2-Horizon-7B}\\[-1pt]
            \textcolor{teal}{\texttt{sft\_1\_final}}\\[-1pt]
            \textcolor{teal}{\texttt{main}}
          }
        & \textcolor{teal}{
            \makecell[l]{
                \texttt{<|ifm|endoftext|>,}\\[-1pt]
                \texttt{<|ifm|im\_end|>}\\[-1pt]
                \texttt{(1, 250019)}
            }
          } \\

        \bottomrule
    \end{tabularx}
\end{table}

\noindent\textbf{Models and data.}
Our main pair uses Qwen3-1.7B-Base as the student and Qwen3-4B in
non-thinking mode as the teacher, and we run the full ablation over the four
termination-handling strategies on this pair. To test whether the mechanism
extends beyond Qwen3, we additionally use Llama-3.2-3B Base and Instruct
\citep{grattafiori2024llama} and Gemma-3-4B PT and IT
\citep{team2025gemma}. Within each pair, the student and the teacher share a
tokenizer and a vocabulary. To observe how termination preferences change
during post-training, we also use K2-Horizon-7B, which releases checkpoints
from pretraining, midtraining, supervised fine-tuning, and the final
post-trained model. We fix the final checkpoint as the teacher and
initialize the student from each of the earlier stages. All models are
trained on DAPO-Math-17K \citep{dapo}. Unless otherwise specified, both
training and evaluation use the TTRL template, which requires the final
answer to be enclosed in \texttt{\textbackslash boxed\{\}}. The template
comparisons additionally use the DAPO template and a raw-question template
that provides only the problem statement. Prompt strings and grader details
are provided in \Cref{app:imple_detail}.

\noindent\textbf{Metrics and budgets.}
During training we track the mean response length and the clipping ratio,
which is the fraction of responses that reach the generation-length budget.
For downstream performance we report Avg@16, which samples 16 responses per
problem, averages correctness over the problems in each benchmark, and then
takes the unweighted mean over AMC23 \citep{maa2023amc}, AIME24
\citep{aime24}, and AIME25 \citep{aime25}. To observe termination behavior
directly, we also compare the termination probabilities of the student and
the teacher at the final position of each student rollout, where we report
both the probabilities of individual tokens and the total stopping mass
$q_\pi(h)=\sum_{e\in\mathcal{E}_{\mathrm{EOS}}}\pi(e\mid h)$. These prefixes
include rollouts that end naturally and rollouts that are truncated at the
budget. The total stopping mass measures how strongly a model prefers to
stop at a given prefix, which is a different quantity from the fraction of
rollouts that terminate naturally. In the main experiments, training allows
at most 1{,}024 prompt tokens and 7{,}168 response tokens, while independent
checkpoint evaluation allows at most 8{,}192 generated tokens, and the
settings for each K2-Horizon stage are described with the corresponding experiments.
The teacher reference lines in our figures report the response length and
clipping statistics of teacher generations on the same prompts, which is a
different quantity from the teacher probabilities that we evaluate at
student prefixes. Remaining optimization and sampling settings are provided
in \Cref{app:exp_details}.

\section{Termination Mismatch and Probability-Level Alignment}
\label{sec:length_inflation}
\label{sec:beyond_qwen}

We use Qwen3 to diagnose termination-token mismatch. We first show that
aligning the decoding stopping set does not remove length inflation, then
analyze how the mismatch enters sampled-token supervision, then compare
three corrections that align termination at the probability level, and
finally test the mechanism on Llama 3.2 and Gemma 3.

\subsection{Decoding EOS Alignment Alone Is Insufficient}
\label{sec:causes}

As shown in \Cref{fig:baseline_opd}, vanilla OPD on Qwen3 progressively
increases response length and clipping, and many rollouts eventually exhaust
the generation budget. The qualitative examples in \Cref{fig:demo} indicate
that this behavior is primarily a failure to terminate, because the student
often produces a correct answer well before the end of the response and then
continues with redundant or repetitive text. \Cref{tab:eos_configs} provides
an immediate clue. The Qwen3 base checkpoint declares \texttt{<|endoftext|>}
as its EOS token, whereas the post-trained checkpoint also recognizes
\texttt{<|im\_end|>} as a conversational termination token
\citep{qwen3concepts}. The student and the teacher therefore share a
tokenizer and a vocabulary, but their declared termination conventions
differ. A first hypothesis is that the rollout decoder does not recognize all
termination tokens that are relevant to this pair, in which case aligning the
decoding stopping set would be sufficient.

\noindent\textbf{Fix 1: Shared-set decoding.}
We first test the most direct correction at the decoding level. For the Qwen3
pair, we register both \texttt{<|endoftext|>} and \texttt{<|im\_end|>} as
valid stopping tokens during student rollout, while the OPD objective and the
token probabilities of both models remain unchanged. As shown in
\Cref{fig:fix_cmp}, Fix~1 follows nearly the same response-length and
clipping trajectory as vanilla OPD and does not remove the failure. Aligning
the decoding stopping set is therefore not sufficient, which leads us to
examine how termination is represented in the distillation signal.
\begin{figure}[!t]
    \centering
    \includegraphics[width=\linewidth]{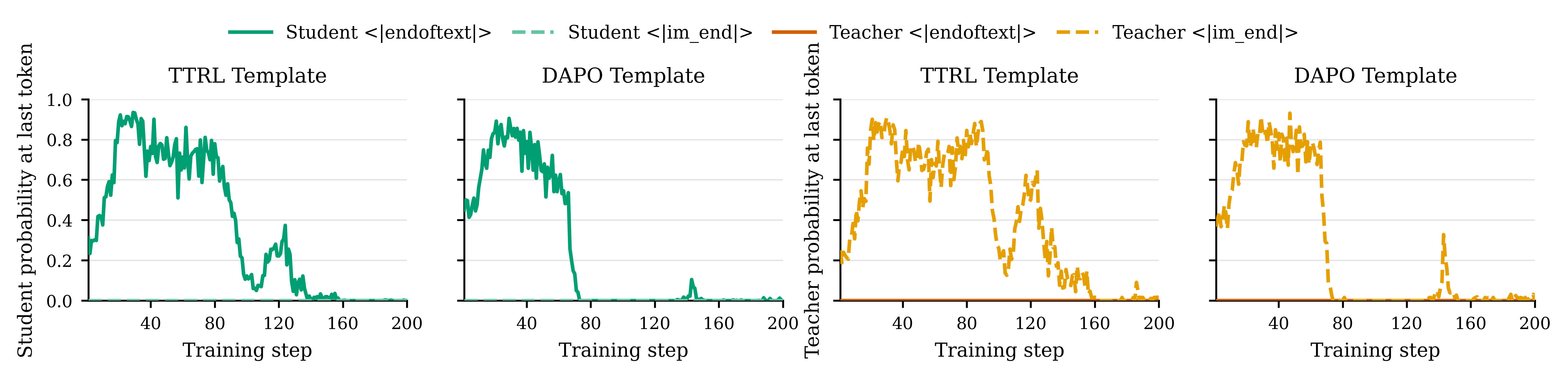}
    \caption{
    \textbf{Qwen3 termination-token probabilities under vanilla OPD.}
    Probabilities are measured at the final generation position of each
    rollout, which either ends naturally or is truncated at the generation
    budget. The student favors \texttt{<|endoftext|>}, while the teacher
    favors \texttt{<|im\_end|>} at the same rollout-end prefixes.
    }
    \label{fig:eos_prob}
\end{figure}

The failure of Fix~1 motivates an examination of the learned distributions
rather than the decoding configuration alone. \Cref{fig:eos_prob} shows a
much stronger mismatch at the probability level. At the same rollout-end
prefixes, the Qwen3 base student places most of its termination probability
on \texttt{<|endoftext|>}, whereas the post-trained teacher favors
\texttt{<|im\_end|>}. The student also assigns very small probability to the
teacher-preferred \texttt{<|im\_end|>} token, approximately $10^{-11}$ around
terminal states. Registering this token as an additional stopping token
therefore has little practical effect, because the student almost never
samples it.

The mismatch also changes the supervision that reaches the termination token
which the student \emph{does} sample. For a termination token $e$ generated
by the student, sampled-token OPD uses
\[
A_t(e)
=
\log \pi^E(e\mid s_t)
-
\log \pi_\theta(e\mid s_t).
\]
This coefficient is negative whenever the teacher assigns less probability to
that particular surface token than the student does, even if the teacher
assigns substantial probability to terminating through another
termination-equivalent token. In this case, sampled-token OPD suppresses the
EOS token that the student sampled instead of recognizing that the teacher
may express the same semantic decision to stop through a different token. For
Qwen3, the student therefore receives a suppressive signal on its native
\texttt{<|endoftext|>} token, and at the same time it receives little direct
supervision on the teacher-preferred \texttt{<|im\_end|>}, which is rarely
sampled. Consistent with this mechanism, the probability that the student
assigns to its native EOS falls from roughly $0.8$ early in training to near
zero in \Cref{fig:eos_prob}, while response length and clipping increase.

The low student probability of the teacher-preferred EOS makes this
asymmetry particularly severe under sampled-token OPD. At each visited
prefix, direct teacher supervision is obtained through the token sampled
by the student. An alternative EOS with negligible student probability
therefore receives very few direct sampled updates, while the student's
native EOS continues to receive the suppressive signal described above.
This makes reliable transfer of the teacher's preferred termination form
difficult at practical sampling budgets, reflecting a sampling-coverage
limitation related to those studied in on-policy optimization
\citep{mei2021understanding,agarwal2021theory}.

Adding the teacher-preferred token to the decoding stopping set does not
address this asymmetry: it neither changes the token-level distillation
signal nor directly increases the probability of sampling that token.
In \Cref{sec:k2_stages}, we examine a complementary case in which the
pretrained student already assigns non-negligible probability to the
teacher-preferred EOS, allowing vanilla OPD to transfer the surface
termination form, although termination later deteriorates. Fix~1 therefore
distinguishes decoding alignment from alignment of the distillation signal:
tokens treated as equivalent stopping events by the decoder are still
supervised as distinct actions by the objective.

\paragraph{Remark: sampling versus objective.}
The difficulty is not solely a consequence of single-token sampling.
Under a softmax parameterization, the full-vocabulary local reverse-KL
gradient is also weighted by the student's token probabilities, so a
teacher-preferred EOS with negligible student probability can receive only
a weak recovery signal even when the entire vocabulary is evaluated.
Full-vocabulary computation removes current-token sampling noise, but does
not remove this probability weighting or reconcile the semantics of
different termination tokens. Because full-vocabulary OPD incurs
substantially higher computational cost, we do not evaluate it here and
restrict our empirical analysis to sampled-token OPD.

\subsection{Probability-Level Termination Alignment}
\label{sec:fixes}
The failure of Fix~1 indicates that the termination mismatch must be
addressed in the distillation signal, so we consider corrections that
reconcile the termination probabilities of the teacher and the student. Let
$\mathcal{E}_{\mathrm{EOS}}$ denote the termination-equivalent tokens under
the rollout protocol of a pair, and let
$e_\star\in\mathcal{E}_{\mathrm{EOS}}$ be a canonical EOS token that the base
student supports. Fix~1 modifies only the decoding interface. We next
consider three ways of aligning the termination signal itself, which modify
the teacher distribution, the semantic action used by the objective, and the
student action space.

\noindent\textbf{Fix 2: Teacher-side EOS mapping.}
We map the probability mass that the teacher places on all
termination-equivalent tokens to the canonical student EOS token:
\[
\widetilde{\pi}^{E}(e_\star\mid s_t)
=
\sum_{e\in\mathcal{E}_{\mathrm{EOS}}}
\pi^E(e\mid s_t).
\]
This idealized mapping sets the teacher probabilities of the other EOS tokens
to zero. In practice we retain negligible probability on them for numerical
stability and adjust the canonical probability so that the total mass is
preserved. The student distribution is unchanged, and rollout terminates only
on $e_\star$, which transfers the stopping signal of the teacher to a token
that the student already supports.

\noindent\textbf{Fix 3: Semantic EOS class.}
Rather than choosing a canonical surface token, we treat all tokens in
$\mathcal{E}_{\mathrm{EOS}}$ as realizations of a single semantic action
$\textsc{stop}$. For either $\pi\in\{\pi_\theta,\pi^E\}$, define
\[
\bar{\pi}(\textsc{stop}\mid s_t)
=
\sum_{e\in\mathcal{E}_{\mathrm{EOS}}}
\pi(e\mid s_t),
\]
while $\bar{\pi}(a\mid s_t)=\pi(a\mid s_t)$ for non-EOS tokens. If the
sampled token $y_t$ is an EOS token, we replace it by
$\bar y_t=\textsc{stop}$, and otherwise $\bar y_t=y_t$. We then apply the
standard sampled-token OPD update directly:
\[
A_t
=
\log \bar{\pi}^{E}(\bar y_t\mid s_t)
-
\log \bar{\pi}_{\theta}(\bar y_t\mid s_t),
\qquad
g_t
=
A_t\nabla_\theta
\log \bar{\pi}_{\theta}(\bar y_t\mid s_t).
\]
EOS samples are thus supervised through the total termination probability,
non-EOS tokens are unchanged, and all tokens in $\mathcal{E}_{\mathrm{EOS}}$
are registered as valid stopping tokens during rollout.

\noindent\textbf{Fix 4: Canonical single-EOS action space.}
We map the EOS probability mass of the teacher to $e_\star$ as in Fix~2, and
we remove all other tokens in $\mathcal{E}_{\mathrm{EOS}}$ from the sampling
distribution of the student, which we renormalize over the remaining
probabilities. The same renormalized distribution is used for the actor
log-probability computation, and rollout terminates only on $e_\star$. This
produces a consistent action space with a single canonical termination
action, irrespective of how many termination-equivalent tokens the original
vocabulary contains.

\begin{figure}[!t]
    \centering
    \includegraphics[width=0.4\linewidth,
        trim=0bp 313.2bp 0bp 7.2bp,clip]{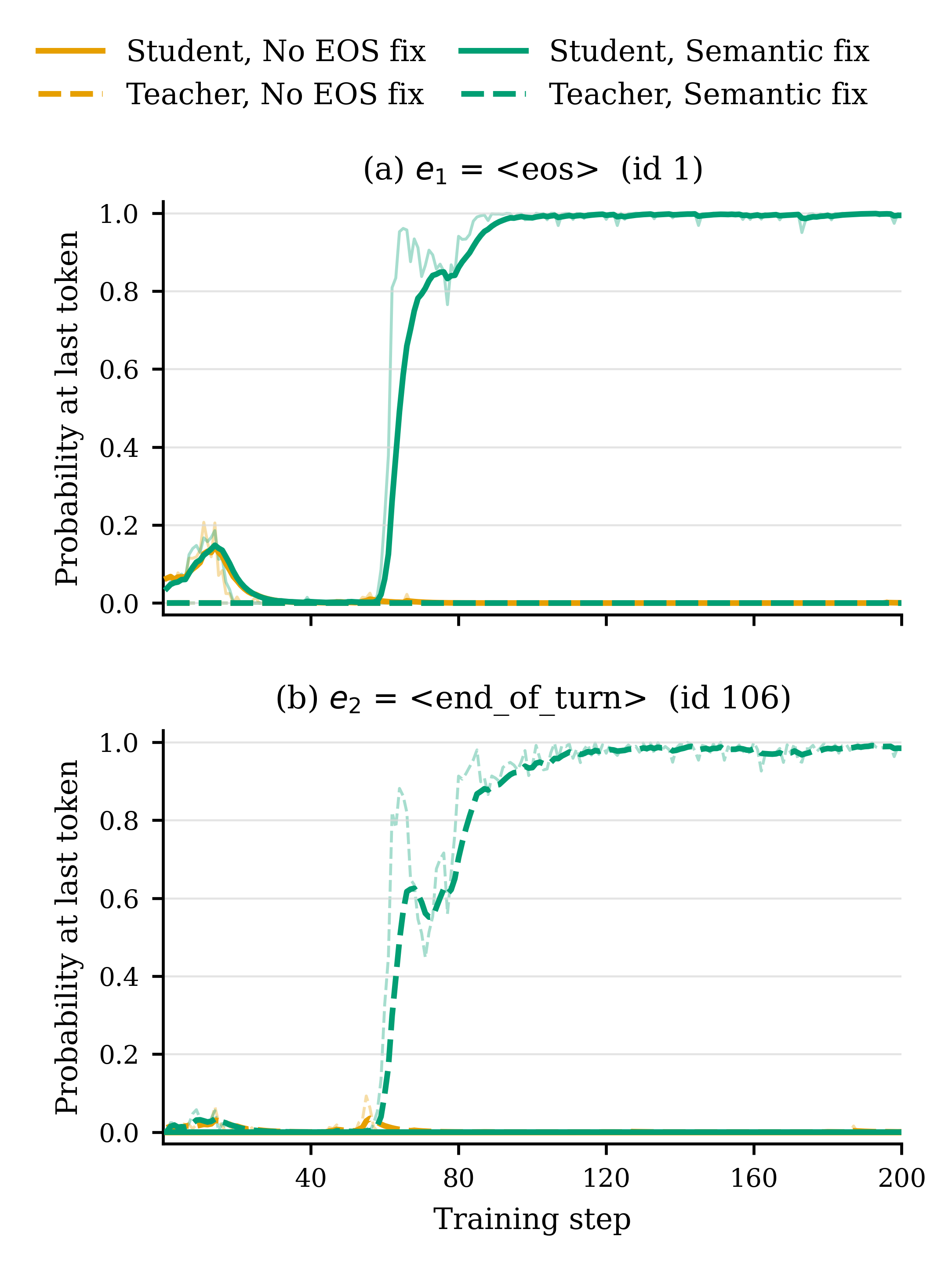}
    \par\vspace{0.2em}
    \begin{minipage}[t]{0.4\linewidth}
        \vspace{0pt}
        \centering
        \setlength{\parskip}{0pt}
        \includegraphics[width=\linewidth,
            trim=0bp 172.8bp 0bp 38.4bp,clip]{fig/gemma_eos_pref.png}%
        \par\nointerlineskip
        \includegraphics[width=\linewidth,
            trim=0bp 4.8bp 0bp 315.96bp,clip]{fig/gemma_eos_pref.png}
    \end{minipage}\hspace{5em}
    \begin{minipage}[t]{0.4\linewidth}
        \vspace{0pt}
        \centering
        \includegraphics[width=\linewidth,
            trim=0bp 4.8bp 0bp 182.76bp,clip]{fig/gemma_eos_pref.png}
    \end{minipage}
    \caption{
    \textbf{Gemma 3 termination preferences within a shared EOS set.}
    Probabilities are measured at the final generation position of each
    rollout, which either ends naturally or is truncated at the generation
    budget. Despite sharing the same declared EOS set, the Gemma 3 PT student
    and IT teacher favor \texttt{<eos>} and \texttt{<end\_of\_turn>},
    respectively. The semantic correction recovers termination without
    forcing these surface preferences to coincide.
    }
    \label{fig:gemma_eos_pref}
\end{figure}

\begin{figure}[!t]
    \centering
    \begin{subfigure}[t]{\linewidth}
        \centering
        \includegraphics[width=\linewidth]{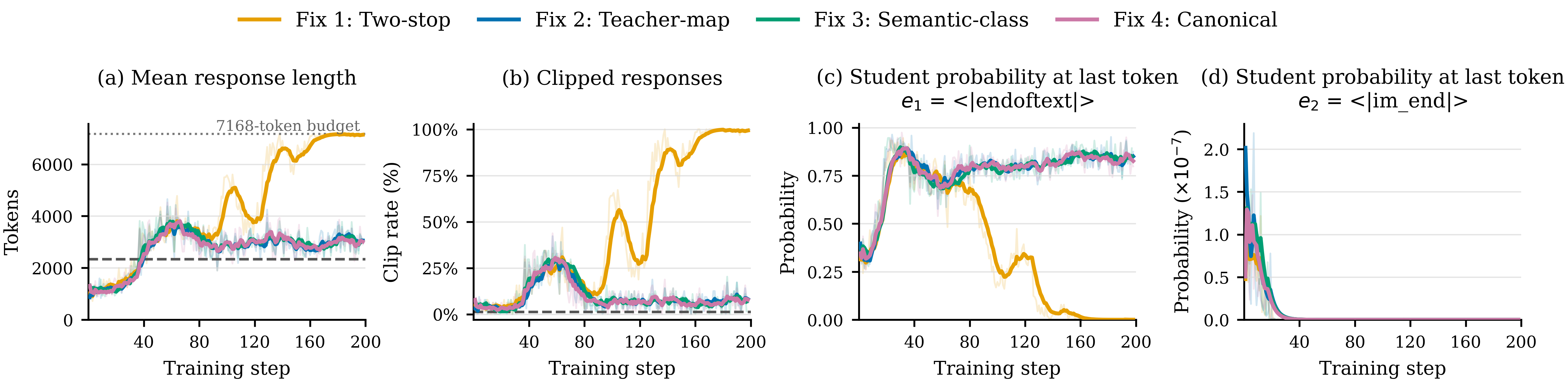}
        \caption{Qwen3: comparison of four EOS corrections.}
        \label{fig:fix_cmp}
    \end{subfigure}

    \vspace{0.6em}
    \begin{subfigure}[t]{\linewidth}
        \centering
        \includegraphics[width=\linewidth]{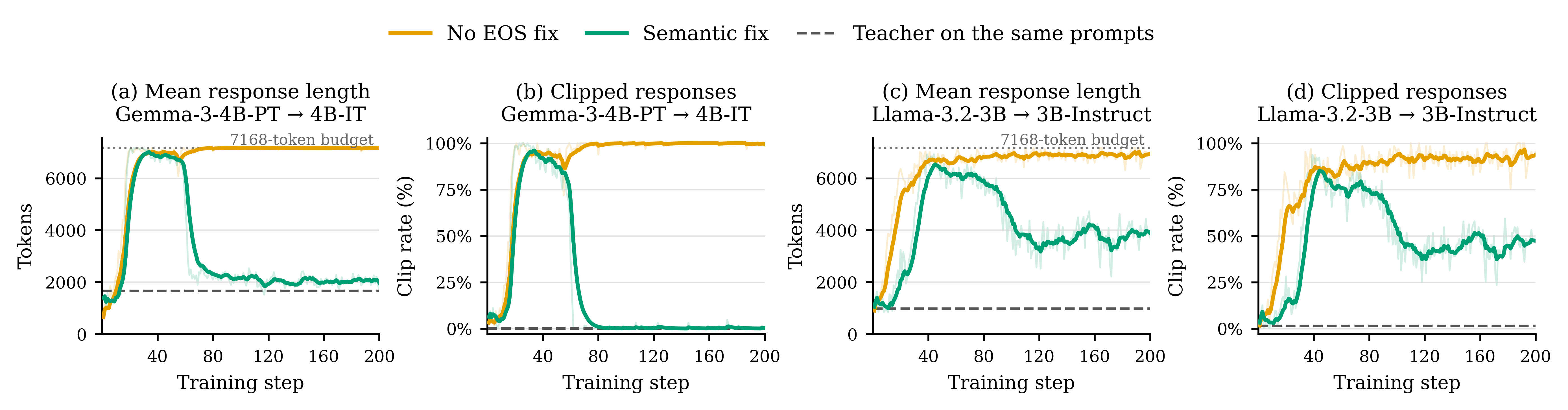}
        \caption{Gemma 3 (left) and Llama 3.2 (right): vanilla OPD versus semantic EOS correction.}
        \label{fig:more_models_length_drift}
    \end{subfigure}
    \caption{
    \textbf{EOS correction across model families.}
    The semantic EOS correction mitigates length inflation in Qwen3 (top),
    Gemma 3 (bottom left), and Llama 3.2 (bottom right), with different
    recovery dynamics and residual gaps from the teachers. The Qwen3 ablation
    additionally compares teacher-side EOS mapping and canonical single-EOS
    mapping, which behave similarly to semantic aggregation, against
    two-stop decoding, which is insufficient. Gray reference lines in the
    length and clipping plots indicate the corresponding teacher statistics
    on the same prompts. All experiments use the TTRL prompt template.
    }
    \label{fig:eos_correction_across_families}
\end{figure}

The ablation in \Cref{fig:fix_cmp} completes the diagnostic sequence in
\Cref{sec:causes}. Fix~1, which changes only the rollout stopping set,
remains close to vanilla OPD, whereas Fixes~2, 3, and~4 reconcile termination
at the probability level and behave similarly: response length and clipping
remain substantially closer to the teacher references, and the native
termination probability of the student no longer collapses. This comparison
isolates the part of the procedure that matters, because adding
\texttt{<|im\_end|>} to the stopping set neither makes the initially
unsupported token likely to be sampled nor changes the negative supervision
that the frequently sampled \texttt{<|endoftext|>} receives. Only a
correction of the termination signal at the probability level addresses the
mismatch directly.

\noindent\textbf{Which correction to adopt.}
Fixes~2, 3, and~4 are empirically similar in the Qwen3 ablation, but they
differ in the assumptions they require. For the Qwen3 pair studied above,
teacher-side EOS mapping (Fix~2) is already sufficient and provides the
simplest correction, because once the corresponding student and teacher
termination tokens are known it modifies only the teacher distribution and
leaves both the student distribution and the sampled-token OPD update
unchanged, while canonical single-EOS decoding (Fix~4) further constrains the
action space of the student. Fixes~2 and~4, however, both require a
particular surface token to be designated as the canonical representation of
termination. This is natural for a known Qwen3 mismatch, but it is less
convenient when the correction is extended to model families with different
termination conventions, and especially when the student and the teacher
expose the same declared EOS set but prefer different tokens within it.
Semantic EOS aggregation (Fix~3) avoids this requirement, because it matches
the total probability assigned to termination without designating a canonical
surface form and without requiring the two models to learn identical token
preferences. We therefore use Fix~2 as the simplest correction for the
explicit Qwen3 mismatch, and we adopt semantic EOS aggregation as the default
formulation for the cross-family experiments in \Cref{sec:cross_family}
(the same base-to-post-trained distillation repeated within Qwen3,
Llama 3.2, and Gemma 3, not distillation between families).

\noindent\textbf{Template and grader effects.}
The shared termination failure should be distinguished from its effect on
measured task performance. The Qwen3 template comparison in
\Cref{fig:baseline_eval} shows that cross-template evaluation can reduce the
observed response length, most clearly when a DAPO-trained model is evaluated
with the TTRL template, and that performance gains from OPD can persist
despite severe length inflation, particularly under TTRL evaluation. The
apparent effect on accuracy also depends on the grader, because DAPO-style
evaluation is more sensitive to long redundant continuations
(\Cref{app:dapo_grader}). Length inflation therefore does not necessarily
imply a comparable loss of reasoning capability. These observations about
templates complement the cross-family termination mechanism rather than
define it.

\subsection{Termination Mismatch Across Model Families}
\label{sec:cross_family}
The EOS configurations in \Cref{tab:eos_configs} illustrate why a
representation-independent formulation is useful. Llama 3.2 resembles Qwen3,
because its base checkpoint declares \texttt{<|end\_of\_text|>} while the
post-trained checkpoint also recognizes conversational termination tokens
such as \texttt{<|eom\_id|>} and \texttt{<|eot\_id|>}. Gemma 3 provides a
complementary and more general case. Its PT and IT checkpoints already
declare the same stopping tokens, \texttt{<eos>} and
\texttt{<end\_of\_turn>}, so no token is missing from the decoding stopping
set. \Cref{fig:gemma_eos_pref} nevertheless shows that their learned
preferences differ sharply: the student assigns its stopping probability
mainly to \texttt{<eos>}, whereas the teacher favors
\texttt{<end\_of\_turn>}, and with the semantic correction the student
recovers reliable termination without adopting the surface token that the
teacher prefers. The relevant mismatch is therefore more general than a
discrepancy between declared EOS sets, since models can expose the same
decoding interface while they distribute their learned termination
probability differently across semantically equivalent tokens.

The cross-family results in \Cref{fig:eos_correction_across_families} support
this interpretation. Relative to vanilla OPD, the semantic EOS correction
substantially reduces response length and clipping in Qwen3, Llama 3.2, and
Gemma 3, where Qwen3 and Gemma 3 approach their corresponding teacher
reference lengths relatively closely and Llama 3.2 retains a larger residual
gap. The shared effect is therefore not that the student is forced to
reproduce the EOS token that the teacher prefers. Semantic aggregation
instead removes the competition between surface tokens by supervising their
total probability as a common stopping action, and the different recovery
trajectories indicate that correcting this mismatch does not imply identical
length dynamics across model families. The mechanism further predicts that
the termination action of the student is suppressed from the beginning of
training when the teacher-preferred termination token lacks sampling support.
The midtraining and SFT checkpoints of K2-Horizon provide the corresponding
control, because they already assign substantial probability to the
teacher-preferred termination token, and vanilla OPD from these checkpoints
shows neither the termination collapse observed above nor unstable response
length and clipping. We develop this contrast in \Cref{sec:k2_stages}, where
we also show that sampling support alone does not keep termination stable
throughout training.

\begin{figure}[t]
    \centering
    \includegraphics[width=\linewidth]{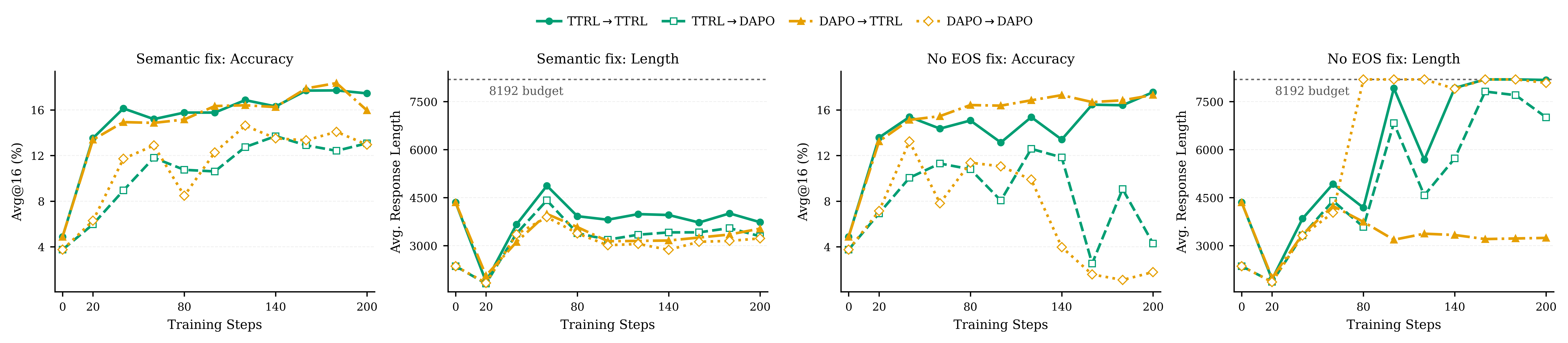}
    \caption{
    \textbf{Template-dependent evaluation in Qwen3 with and without EOS correction.}
    We distill Qwen3-4B into Qwen3-1.7B-Base on the DAPO dataset and
    report Avg@16 over AIME24, AIME25, and AMC23, computed using 16 sampled
    responses per problem. In the legend,
    $\mathrm{A}\!\rightarrow\!\mathrm{B}$ denotes that the student is trained
    with template A and evaluated with template B.
    }
    \label{fig:baseline_eval}
\end{figure}



\section{Going Beyond the EOS Mismatch}
\label{sec:k2_stages}

It was demonstrated in \Cref{sec:length_inflation} that aligning termination at the
probability level removes the length inflation caused by the mismatch. This
section addresses two questions that the mismatch does not answer: under
which conditions the mismatch is harmful, and which length dynamics remain
after the correction.


\subsection{OPD Across Post-Training Stages}

The experiments above compare base students with post-trained teachers.
To examine how termination behavior depends on the student's training
stage, we use K2-Horizon 7B, which includes checkpoints spanning pretraining,
midtraining, supervised fine tuning, and the final post-trained model.
We fix the final checkpoint as the teacher and initialize the student from
different earlier stages.

K2-Horizon provides a useful stage-wise comparison because the relevant
termination-token IDs and the declared stopping set remain unchanged across
stages (\Cref{tab:eos_configs}). We can therefore examine how learned
termination preferences and OPD dynamics evolve over training without
conflating them with changes in the declared EOS configuration.

\Cref{fig:k2_stage_alignment} shows a clear stage-dependent shift in
termination preference. The pretrained student primarily favors
\texttt{<|ifm|endoftext|>}, whereas the final post-trained teacher favors
\texttt{<|ifm|im\_end|>}. Most of this shift already occurs during
midtraining: the midtraining checkpoint places substantial probability on
the teacher-preferred termination token, and by the SFT stage the student's
termination preference is already close to that of the final teacher.

The Pretrain-to-Final run also provides a useful contrast with Qwen3.
Although the pretrained K2-Horizon student and final teacher initially prefer
different termination tokens, the student already assigns non-negligible
probability to the teacher-preferred token. It can therefore sample this
token and receive direct teacher supervision on it. Under vanilla OPD,
termination mass consequently shifts toward the teacher-preferred surface
form, showing that sampled-token OPD can transfer an alternative termination
token when it already has sufficient support under the student policy.

Importantly, this transfer does not prevent a later stopping failure. After
the teacher-preferred termination form has been learned, termination
probability decreases again and response length grows toward the generation
budget. Thus, surface-token alignment alone does not account for the full
Pretrain-to-Final dynamics. We return to this late-stage behavior in
\Cref{sec:residual_dynamics}.

\noindent\textbf{A non-interference check at the midtraining and SFT stages.}
The midtraining and SFT checkpoints provide a control for the semantic
correction.
The midtraining and SFT students already place substantial probability on the
teacher-preferred termination token, so vanilla OPD from these checkpoints
does not exhibit the severe termination failure observed in the more
mismatched base-to-post-trained settings. Applying semantic EOS aggregation
in these cases produces similar response-length, clipping, and termination
dynamics, with no consistent degradation in downstream performance. When the
termination representations are already sufficiently aligned, the correction
is therefore largely inert rather than a force that pushes the student toward
different behavior. We provide the full training and evaluation comparisons
in \Cref{app:k2_later_stage_semantic}.

\subsection{Residual Length Dynamics}
\label{sec:residual_dynamics}

\begin{figure}[!t]
    \centering
    \includegraphics[width=\linewidth]{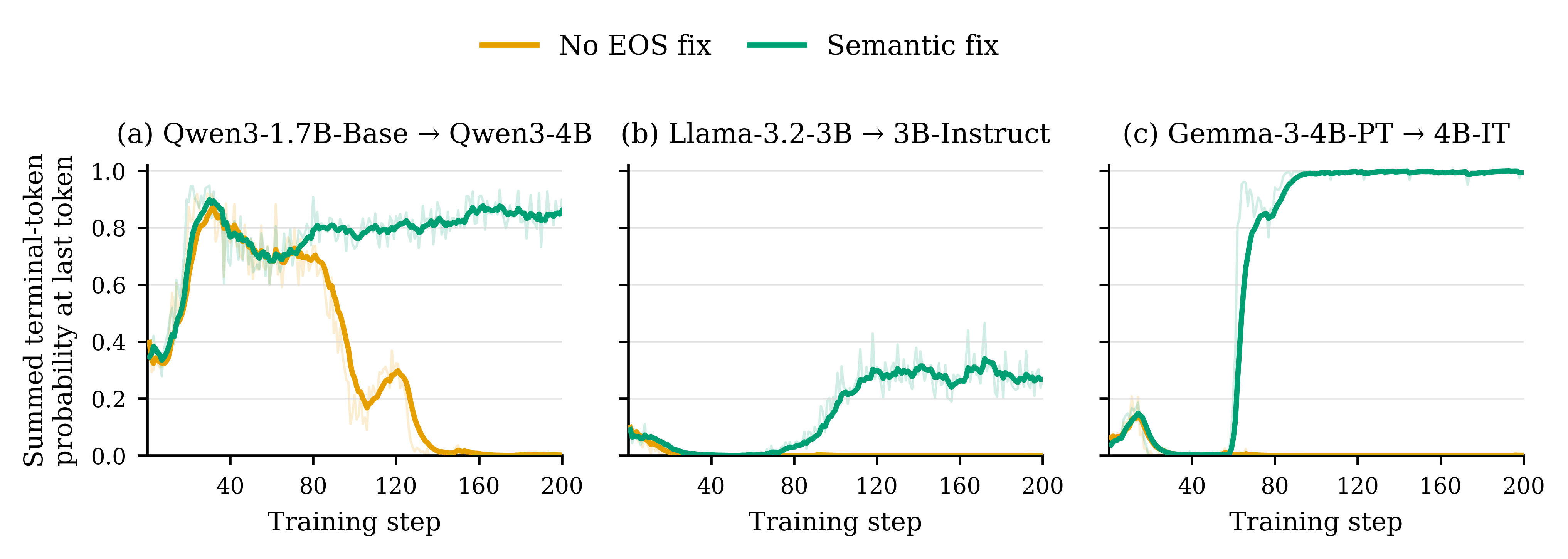}
    \caption{
    \textbf{Evolution of termination probability across model families.}
    Summed probability assigned by the student to its termination tokens at
    the final observed rollout position for (a) Qwen3, (b) Llama 3.2, and
    (c) Gemma 3. Under vanilla OPD, termination probability falls to near zero
    in all three settings. The semantic EOS correction prevents this collapse
    or enables recovery, but the dynamics differ across families: Qwen3
    quickly reaches and maintains a high termination probability, whereas
    Llama 3.2 and Gemma 3 remain in a low-probability regime for longer and
    recover only later in training.
    }
    \label{fig:model_eos_cmp}
\end{figure}

\begin{figure}[t]
    \centering
    \includegraphics[width=\linewidth]{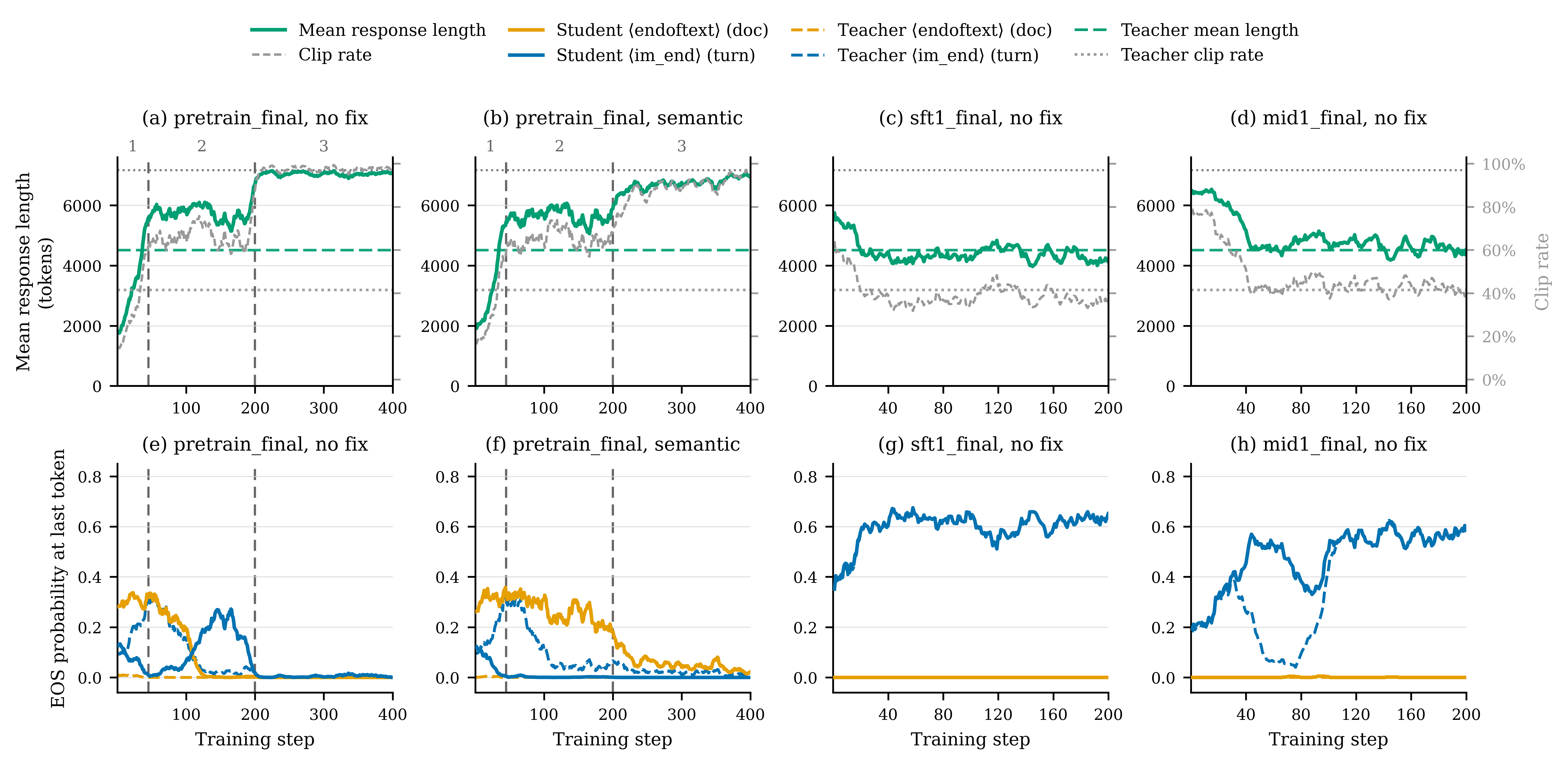}
 \caption{
\textbf{OPD from different K2-Horizon training stages to the final
post-trained model.}
The top row shows response length and clipping rate, and the bottom row
shows student and teacher probabilities on the two termination tokens at
the final rollout position. We compare Pretrain-to-Final distillation with
and without semantic EOS correction, together with vanilla OPD initialized
from the SFT and midtraining checkpoints. Horizontal green dashed and gray
dotted lines denote the teacher mean length and clipping rate, respectively.
Labels 1--3 mark the initial growth, intermediate plateau, and late-stage
re-inflation phases of the Pretrain-to-Final runs; vertical gray dashed
lines indicate their approximate boundaries.
}
    \label{fig:k2_stage_alignment}
\end{figure}

Removing the surface EOS mismatch does not eliminate all response-length
dynamics. We distinguish early length growth after correction from the
renewed inflation observed late in the K2-Horizon run (Phase~3 in
\Cref{fig:k2_stage_alignment}).

\paragraph{Early length growth after EOS correction.}
Even with semantic EOS correction, response length initially increases
before decreasing across Qwen3, Llama 3.2, and Gemma 3
(\Cref{fig:eos_correction_across_families}). This transient drift is weak
in Qwen3 and more pronounced in Llama 3.2 and Gemma 3. Student responses
also remain longer than the corresponding teacher reference at later
checkpoints, with the largest residual gap in Llama 3.2.

The termination dynamics in \Cref{fig:model_eos_cmp} are consistent with
these differences. With semantic correction, Qwen3 recovers its stopping
probability relatively quickly, whereas Llama 3.2 and Gemma 3 remain in a
low-termination regime for longer. These curves describe different recovery
dynamics, but do not by themselves establish their cause.

The corrected K2-Horizon Pretrain-to-Final run also exhibits an initial increase
in response length, marked as Phase~1 in \Cref{fig:k2_stage_alignment}.
Unlike the subsequent recovery seen in the other families, this increase
is followed by an elevated, fluctuating plateau in Phase~2. The plateau
should therefore not be interpreted as recovery to the teacher's response
length or as convergence of the termination distributions.

In \Cref{sec:length_drift}, we analyze a possible contributor to early
length growth from the trajectory-level reverse KL perspective. Local
sampled-token OPD omits the future student--teacher mismatch incurred by
continuation. At a fixed prefix, this omission removes a nonnegative
contribution to the stopping-logit update relative to the
trajectory-consistent reverse KL update. The resulting relative
continuation bias is consistent with early length growth after EOS
correction, including Phase~1 of the corrected K2-Horizon run.

\paragraph{Late-stage re-inflation in K2-Horizon.}
Phase~3 of the Pretrain-to-Final runs reveals an additional failure.
After the intermediate plateau, response length increases again,
termination probabilities approach zero, and almost all rollouts
eventually reach the generation budget. This transition is relatively
abrupt under vanilla OPD and more gradual under semantic EOS correction.

In the vanilla run, the renewed inflation occurs after naturally
terminated rollouts have already shifted toward the teacher-preferred
EOS token. More importantly, the corrected run exhibits the same
qualitative late-stage behavior even though the objective already treats
both EOS tokens as a shared stopping action. Surface termination-token
identity alone therefore cannot explain this renewed collapse.

The late-stage increase is qualitatively similar to the length inflation
and truncation collapse reported by \citet{demisify}, but this similarity
does not establish a shared mechanism. We leave the origin of this
late-stage K2-Horizon collapse, together with the differences in recovery
dynamics across model families, for future work.
\section{Related Works}
\paragraph{Length inflation in on-policy distillation.}
On-policy distillation trains the student on its own sampled rollouts under
per-token supervision from a teacher \citep{lu2025onpolicydistillation,
yang2025qwen3}. Uncontrolled length growth during OPD has been reported in
several recent works, each identifying a distinct mechanism.
\citet{demisify} observe abrupt length inflation accompanied by repetition
saturation, and attribute it to an implicit preference of the distillation
objective for long and repetitive rollouts. \citet{li2026rethinking} find that
overly large generation budgets lead to late-stage training collapse, and
attribute it to the degradation of teacher supervision quality with trajectory
depth. \citet{zhao2026decoupling} relate length inflation to entropy collapse induced by reverse-KL distillation. These mechanisms are not mutually exclusive, and each may well be operative in the settings where it was identified. They do, however, share a precondition: the student must first fail to terminate. In this work we revisit that precondition across base--post-trained model pairs from
several widely used model families, and find that the two models frequently encode termination through different special tokens. Where this is the case, the
token-level objective penalizes the student's stopping action from the outset, producing the same truncation-and-repetition signature studied above but through
a cause that precedes those mechanisms and is removed by a change to the objective alone. Our results are therefore complementary to prior analyses rather
than in tension with them: in settings where teacher and student already agree on termination, the mechanisms above remain the relevant explanation, and where they do not, termination alignment should be established before length inflation is attributed to an algorithmic cause.

\noindent\textbf{Rollout quality degradation and its mitigation.}
Beyond length inflation, a second line of work focuses on the degradation of
rollout quality itself. Once the student errs or begins repeating midway
through a trajectory, its subsequent prefixes drift away from the teacher's
normal distribution; the teacher's preferences over these degenerate prefixes
become decoupled from task correctness, the distillation signal grows
increasingly misaligned, and this in turn pushes the student toward even more
degenerate generations. \citet{liu2026prefix} characterize this as
\emph{local teachability collapse}, and \citet{posbias2026} show that
late-trajectory tokens systematically drift off the teacher distribution.
Two families of remedies have been proposed. The first distills only on the
early segment of a trajectory, where teacher supervision remains reliable, as
in Prefix OPD \citep{prefixopd2026}, Prune-OPD \citep{pruneopd2026}, and
ADWIN \citep{adwin2026}, which adapts the rollout horizon according to prefix
drift. The second improves rollout quality directly so that the student avoids
entering degenerate segments, as in SWITCH \citep{switch2024}, which injects
teacher tokens where the teacher and student disagree, the teacher-blended
behavior policy of \citet{plyusov2026trust}, and Relay-OPD
\citep{relayopd2026}, which lets the teacher take over at the point where the
student's reasoning first goes wrong. All of these works treat rollout
degradation as an algorithmic pathology arising from the interaction between
the distillation objective and on-policy sampling. We instead show that the
student's inability to stop as training proceeds has a lower-level,
implementation-level root cause, namely a termination-token mismatch, which
alone suffices to produce budget saturation without any trajectory-level
pathology.

\noindent\textbf{Termination-token semantic mismatch.}
The Qwen3 base model declares \texttt{<|endoftext|>} as its termination token, whereas the post-trained model uses \texttt{<|im\_end|>} as its primary conversational termination token while retaining \texttt{<|endoftext|>} as an additional stopping token
\citep{qwen3concepts}.

Prior work has identified related special-token issues in OPD. \citet{fu2026revisiting} identify tokenizer and special-token mismatch, including EOS markers, as a failure mode of sampled-token supervision and mitigate it by masking problematic special tokens from the distillation signal. \citet{plyusov2026trust} address the Qwen3 EOS mismatch through EOS canonicalization, mapping the native student and teacher EOS tokens to a shared stopping event before sampling and KL evaluation.

Our work instead takes termination mismatch itself as the object of study. We systematically trace how token-level OPD can amplify a disagreement between surface termination tokens into termination-probability collapse and length inflation, and compare multiple interventions that separately modify the decoding rule, teacher distribution, semantic action space, and student action space. This comparison shows that aligning the decoding stopping set alone is insufficient, whereas aligning termination at the probability level consistently mitigates the failure.

Moreover, we show that the phenomenon is not specific to Qwen3 or to an
explicit difference in declared EOS configurations. Across Qwen3, Llama,
and Gemma, student and teacher models can prefer different surface
tokens for the same semantic stopping action; in Gemma 3, this mismatch
arises even when the two models already declare the same EOS set.
These results motivate treating termination as a semantic event rather
than relying solely on model-level EOS configuration.
\section{Limitations and Future Work}
Our experiments across Qwen3, Llama, and Gemma identify termination-token
mismatch as an important source of length inflation, but not a complete
explanation of the phenomenon. The mismatch can arise either from different
declared stopping sets or from different learned preferences among
termination tokens even when the declared sets coincide. Our K2-Horizon
experiments further show that, after the student has successfully acquired
the teacher-preferred termination token, the probabilities of both
termination tokens can decrease again later in training, accompanied by a
second increase in response length; a similar late-stage increase remains
under semantic EOS correction. We therefore do not attribute these residual
dynamics to termination-token mismatch and leave their underlying mechanism
for future work. Consistent with this limitation, prior work also reports
length inflation in settings without the same declared stopping-set mismatch
\citep{demisify}. Characterizing these remaining mechanisms and the
conditions under which they dominate is an important direction for future
study.

A second limitation concerns the extension beyond single-turn tasks.
Our experiments consider single-turn mathematical reasoning, where the tokens
grouped into the semantic EOS class are operationally equivalent under the
decoding protocol: each ends the response. This equivalence need not hold in
multi-turn or agentic settings, where special tokens may distinguish message
completion, the end of an assistant turn, document termination, and a handoff
to a tool. The base-to-post-trained pairing used here would also not carry over:
most base models lack tool-use and multi-turn capability, so a fair
comparison would instead require two post-trained models of different
sizes, which is a substantially different experimental design. Treating these functionally distinct events as interchangeable
alignment targets could alter interaction control flow rather than merely
correct stopping behavior. Extending our approach therefore requires
context-dependent termination classes that preserve the semantics of the
underlying interaction protocol. Length inflation may also have more severe
consequences in these settings: repetition can exhaust the context budget or
disrupt subsequent tool calls, affecting task success rather than token count
alone. These effects are not evaluated in our current experiments.

Finally, our prompt-template analysis is limited to the Qwen3 experiments
with the three templates and two graders studied here. Within this scope,
the severity of length inflation depends strongly on the training and
evaluation templates, while the associated downstream performance gains can
transfer across templates. The cross-family evidence for termination mismatch
does not, by itself, establish that these template effects generalize to
Llama or Gemma. A more systematic analysis across model families, task domains,
and evaluation protocols is needed to distinguish template-specific generation
artifacts from more general changes in reasoning capability.
\section{Conclusion
}

We document and analyze length inflation in on-policy distillation. One major cause is a termination mismatch: base students and post-trained teachers mark the end of a response with different special tokens. The token-level distillation objective therefore treats the student's own termination as an over-produced token and suppresses it, so the student's termination probability decays over training and responses keep growing. We compared four fixes and found that aligning the two termination tokens into a single stop action at the probability level restores response length and clipping ratio to the teacher's reference levels, whereas merely enlarging the decoding stop set does not. This confirms from the opposite direction that the problem lies in the training objective rather than the decoding setup. Our template experiments further show that the training template has a persistent effect on response length, while the evaluation format substantially changes measured accuracy, and the two need not move together. Overall, termination semantics is an overlooked but important axis of alignment in on-policy distillation, and future work will extend to multi-turn and agentic settings and examine other mechanisms behind length inflation.

\section*{Author Contributions and Acknowledgments}
Four student authors (YY, TY, SL, KZ) contribute equally to this project: The length inflation is observed by SL, SL proposed several mechanisms and solutions to mitigate this. However, the actual EOS mismatch is identified by YY in adjacent projects supervised by WZ in 2026 summer. YY verified this with SL on Qwen. SL, TY and KZ extends this to Gemma, Llama and check the staged phenomena in K2-Horizon.

All authors appreciate the valuable discussions with Dr. Mikhail Yurochkin and Dr. Hector Liu in checking the implementation on their K2-Horizon model.
\bibliography{main}
\bibliographystyle{abbrvnat}

\appendix
\newpage
\section*{Appendix}
\addcontentsline{toc}{section}{Appendix}
\startcontents[sections]
\printcontents[sections]{l}{1}{\setcounter{tocdepth}{2}}

%
%
%

\section{Experimental Details}
\label{app:exp_details}

\paragraph{Implementation and compute.}
Our implementation is built on the public OPD codebase of \citet{li2026rethinking}, with modifications for the EOS corrections and evaluation protocols studied in this work. Unless otherwise specified, all experiments can be run on four NVIDIA RTX PRO 6000 GPUs with 96 GB of memory per GPU. We will release the code and experiment configurations to facilitate reproduction.

\paragraph{Main setup.} Unless otherwise specified, the main experiments in this note use Qwen3-1.7B-Base as the student and Qwen3-4B as the teacher, trained on
DAPO-Math-17K \citep{dapo} with thinking disabled. The default setting uses the TTRL
format for both training and evaluation. The exact prompt templates are
provided in \Cref{app:prompt_templates}. The teacher remains fixed
throughout training.

We use sampled token OPD. Student responses are generated on policy, and
the distillation signal at each generated token is given by the teacher
and student log probability difference on that token. Each training step
contains 16 prompts with four sampled responses per prompt, corresponding
to 64 trajectories per step. We train for 200 steps and save checkpoints
every 20 steps.

\begin{table}[h]
    \centering
    \small
    \caption{Training hyperparameters used in the main experiments.}
    \label{tab:training_hparams}
    \begin{tabular}{ll}
        \toprule
        \textbf{Hyperparameter} & \textbf{Value} \\
        \midrule
        Rollout temperature & 1.0 \\
        Rollout top-$p$ & 1.0 \\
        Rollout top-$k$ & None \\
        Optimizer & AdamW \\
        Learning rate & $1\times10^{-6}$ \\
        Learning rate schedule & Constant \\
        Adam betas & $(0.9,\,0.999)$ \\
        Weight decay & 0.01 \\
        Gradient clipping & 1.0 \\
        Optimization epochs per batch & 1 \\
        Mini batch size & 16 \\
        Loss aggregation & Token mean \\
        \bottomrule
    \end{tabular}
\end{table}

\paragraph{EOS correction experiments.}
The EOS correction experiments use the same training configuration and
differ only in their treatment of termination actions. We compare vanilla
OPD, two stop decoding, teacher side EOS mapping, semantic EOS aggregation,
and canonical single EOS mapping, as described in \Cref{sec:fixes}.

\paragraph{Checkpoint evaluation.}
We independently evaluate checkpoints from steps 20 through 200 on AIME24,
AIME25, and AMC23. These benchmarks contain 30, 30, and 40 problems,
respectively. Each problem is sampled 16 times. Evaluation uses a fixed sampling seed of 0 for all methods and checkpoints. Unless otherwise specified, evaluation uses the TTRL format with thinking disabled.

\begin{table}[t]
    \centering
    \small
    \caption{Sampling hyperparameters used for checkpoint evaluation.}
    \label{tab:eval_hparams}
    \begin{tabular}{ll}
        \toprule
        \textbf{Hyperparameter} & \textbf{Value} \\
        \midrule
        Samples per problem & 16 \\
        Temperature & 0.7 \\
        Top-$p$ & 0.95 \\
        Top-$k$ & None \\
        Repetition penalty & 1.0 \\
        \bottomrule
    \end{tabular}
\end{table}

\paragraph{Evaluation metrics.}
For each problem, we compute the mean correctness over its 16 sampled
responses. We first average over problems within each benchmark and then
report the unweighted macro average over AIME24, AIME25, and AMC23 as
EOS termination rate, format compliance, and answer extraction success
when relevant.

\paragraph{Generation lengths.}
For the Qwen3, Llama 3.2, and Gemma 3 experiments, training uses a maximum
prompt length of 1{,}024 tokens and a maximum response length of 7{,}168
tokens, with a maximum model context of 8{,}192 tokens. Independent
checkpoint evaluation allows up to 8{,}192 generated tokens with a maximum
model context of 12{,}288 tokens.

\paragraph{K2-Horizon stage-wise experiments.}
For the K2-Horizon experiments, training uses a maximum prompt length of
1{,}024 tokens and a maximum response length of 7{,}168 tokens, with a
maximum model context of 8{,}192 tokens, matching the native context length
of the pretrained checkpoint. Checkpoint evaluation allows up to 7{,}168
generated tokens. The Pretrain-to-Final runs are trained for 400 steps,
whereas the Midtrain-to-Final and SFT-to-Final runs are trained for
200 steps. Unless otherwise specified, the remaining OPD configuration
follows the main experimental setup.

\section{Prompt and Evaluation Protocols}

\label{app:imple_detail}
\subsection{Prompt Templates}
\label{app:prompt_templates}

We consider three prompt templates for generating student rollouts.
The exact templates used in our experiments are provided below.
The placeholder \texttt{\{problem\}} is replaced by the corresponding
problem statement.

\begin{promptbox}{DAPO Template}
Solve the following math problem step by step. The last line of your response
should be of the form Answer: $Answer (without quotes) where $Answer is the
answer to the problem.

{problem}

Remember to put your answer on its own line after "Answer:".
\end{promptbox}

\begin{promptbox}{TTRL Template}
{problem} Please reason step by step, and put your final answer within \boxed{}.
\end{promptbox}

\begin{promptbox}{Raw-Question Template}
{problem}
\end{promptbox}
\subsection{Evaluation Graders}
\label{sec:grader_setup}

We evaluate responses under two answer formats corresponding to the DAPO and TTRL templates. The raw-question template is used only for training, as it does not prescribe a canonical final-answer format for reliable automatic extraction. For TTRL-style evaluation, we extract the final \texttt{\textbackslash boxed\{\}} answer. For DAPO-style evaluation, we start from the official DAPO grader, which extracts the answer following the final \texttt{Answer:} marker within the last 300 characters of the response and applies normalized exact matching.

We find, however, that the unmodified DAPO grader substantially underestimates the performance of Qwen3 models because it is sensitive to equivalent Markdown layouts of the final answer. As shown in \Cref{fig:dapo_eval_cmp}, the Qwen3-4B teacher obtains only 6.28\% Avg@16 under the original grader, compared with 24.34\% after applying our format-compatible parsing extension. Under the original grader, several student checkpoints even appear to outperform the teacher, despite the large discrepancy being primarily attributable to answer formatting rather than mathematical correctness.

Manual inspection shows that Qwen3 frequently places the answer after \texttt{Answer:} using display-math formatting, including multiline \texttt{\$\$...\$\$} blocks, which can be missed by the original line-based extraction. We therefore use a minor Qwen3-compatible extension of the DAPO grader throughout our DAPO-format evaluations. The extension changes only answer parsing; the original DAPO normalization and exact-match criterion are retained. Details are provided in \Cref{app:dapo_grader}.

\paragraph{Why answer format matters more in OPD.}
This issue is particularly relevant when evaluation protocols designed for RLVR are reused for OPD. In RLVR, a format sensitive verifier or reward provides direct training pressure toward outputs that can be recognized by the corresponding grader. In OPD, however, no such outcome reward is present. The student's output format is instead determined mainly by the teacher distribution together with the prompt. Consequently, a grader that assumes a particular answer layout may underestimate performance when that layout differs from the format naturally preferred by the teacher. This motivates using an evaluation parser that preserves the original correctness criterion while accommodating equivalent answer layouts.
\begin{figure}[t]
    \centering
    \includegraphics[width=\linewidth]{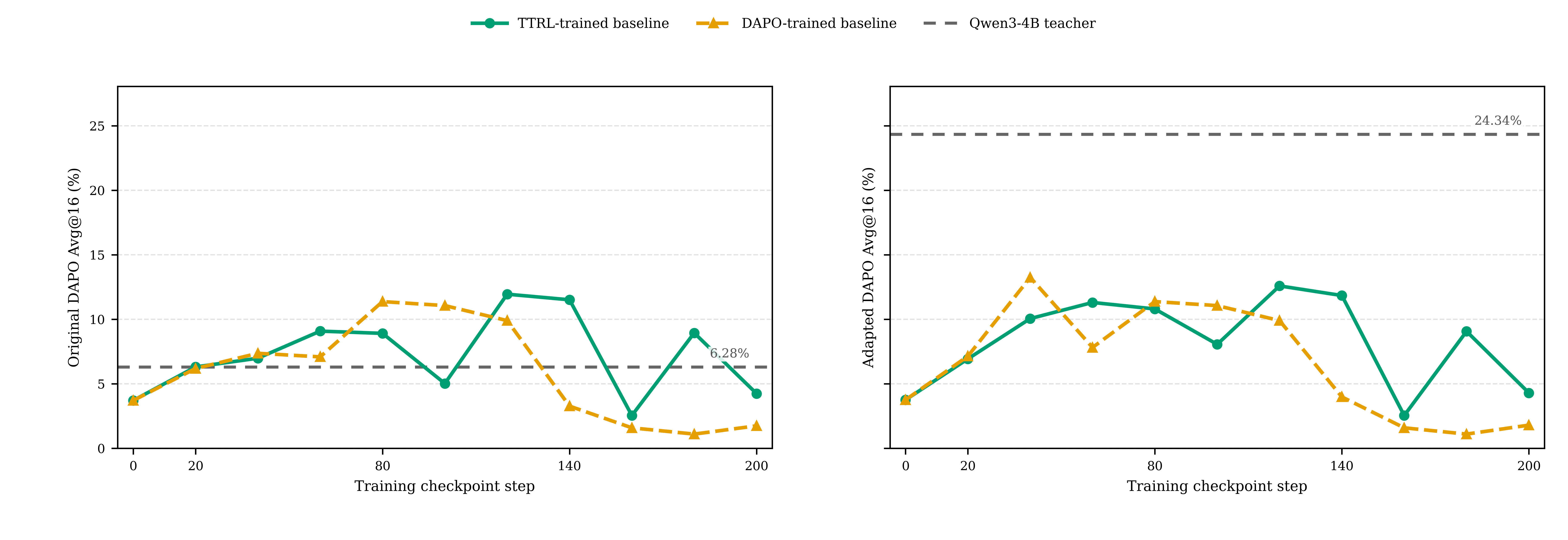}
    \caption{
    \textbf{Effect of DAPO answer parsing on Qwen3 evaluation.}
    We compare the original DAPO grader (left) with our
    format-compatible extension (right) for vanilla OPD using
    Qwen3-1.7B-Base as the student and Qwen3-4B as the teacher.
    The dashed horizontal line denotes teacher performance under the
    corresponding grader. The teacher score increases from 6.28\% to
    24.34\% Avg@16 when equivalent Qwen3 LaTeX answer layouts are
    parsed correctly, indicating that the original grader can
    substantially underestimate Qwen3 performance due to formatting.
    }
    \label{fig:dapo_eval_cmp}
\end{figure}

\subsection{DAPO-Compatible Evaluation Grader}
\label{app:dapo_grader}

The official DAPO grader considers only the final 300 characters of a
response, extracts the answer following the last case-insensitive
\texttt{Answer:} marker, applies the official DAPO normalization to the
prediction and ground truth, and performs normalized exact matching.

Qwen3 frequently renders valid answers using display-math formatting,
for example
\[
\texttt{Answer: \$\$\textbackslash boxed\{204\}\$\$},
\]
or with the display-math block beginning on the following line. Such outputs can be misparsed by the original line-based extraction.

We therefore first apply the unmodified official DAPO scorer and invoke
a fallback parser only when it returns an incorrect result. The fallback
uses the same final-300-character window and still requires the answer
to follow the final \texttt{Answer:} marker, but additionally supports
equivalent inline and display LaTeX layouts, including multiline
\texttt{\$\$...\$\$} blocks. After extraction, we apply exactly the same
DAPO normalization and exact-match comparison as the official grader.
No semantic matching, symbolic equivalence, or search for answers
elsewhere in the response is introduced.

\section{Future Mismatch and Transient Length Drift}
\label{sec:length_drift}

After correcting the termination-token mismatch, we no longer observe
the persistent budget saturation seen in vanilla OPD. However,
response length can still increase during the early stage of training
and decrease again later, a pattern that is particularly pronounced
in our Llama and Gemma experiments.

Empirically, at early steps we often observe low termination
probabilities near the end of a rollout, with the teacher assigning
even less probability to stopping than the student at many such
prefixes. At these prefixes, the local OPD coefficient for a sampled
stopping action is negative. This does not by itself determine the
expected stopping-logit update, which also depends on the
coefficients assigned to continuation tokens. Moreover, the local
signal does not account for the mismatch encountered after
continuing along the student's future trajectory.

This distinction is closely related to MiniLLM
\citep{minillm}, which studies sequence-level reverse KL
distillation through policy gradient optimization and decomposes
the gradient into single-step and long-range contributions.
Here, we use the unnormalized trajectory-level gradient as an
analytical reference for local sampled-token OPD, rather than
identifying it with MiniLLM's final length-normalized update.
We characterize a successor-KL contribution omitted by the local
update and show that its stopping-logit component is nonnegative
at a fixed prefix.

For this analysis, assume that the teacher is fixed and that both
policies are defined on the same action space, with aligned
termination semantics represented by a single stopping action $e$.
We use a common finite generation budget $H$ and assume common
support and finite KL divergences. At position $t$, let
$s_t=(x,y_{<t})$, whose prefix length determines the remaining
budget, and let $f(s,y)$ denote the next state after predicting
token $y$. Both predicting $e$ and exhausting the generation
budget lead to an absorbing terminal state $\dagger$.
We analyze expected on-policy, unclipped score-function updates
with trajectory-summed losses, before optimizer preconditioning;
the comparison need not hold exactly for updates with stale
rollouts, clipping, or length-dependent reweighting.

Define the local reverse KL cost
\[
\rho_\theta(s,y)
=
\log\frac{\pi_\theta(y\mid s)}{\pi^E(y\mid s)},
\]
and let
\[
V_\theta(s)
=
D_{\mathrm{KL}}
\left(
P_{\pi_\theta}^{s}
\,\|\,
P_{\pi^E}^{s}
\right)
\]
denote the reverse KL between the student and teacher
distributions over the remaining sampled tokens up to stopping
or budget exhaustion, with $V_\theta(\dagger)=0$.
By the chain rule of KL,
\begin{equation}
\label{eq:future_kl_bellman}
V_\theta(s)
=
\mathbb{E}_{y\sim\pi_\theta(\cdot\mid s)}
\left[
\rho_\theta(s,y)
+
V_\theta(f(s,y))
\right].
\end{equation}

\paragraph{Local OPD omits the future mismatch.}
Let $J_x(\theta)=V_\theta(s_1)$ be the trajectory-level reverse KL
for prompt $x$. For a student trajectory of length $T\le H$, its
negative gradient can be written as
\[
-\nabla_\theta J_x(\theta)
=
\mathbb{E}_{\tau\sim P_{\pi_\theta}^{s_1}}
\left[
\sum_{t=1}^{T}
A_t^{*}
\nabla_\theta\log\pi_\theta(y_t\mid s_t)
\right],
\]
where
\begin{equation}
\label{eq:trajectory_opd_signal}
A_t^{*}
=
\log\frac{\pi^E(y_t\mid s_t)}
         {\pi_\theta(y_t\mid s_t)}
-
V_\theta(f(s_t,y_t)).
\end{equation}
Here and below, the scalar coefficient multiplying the policy
score is treated as detached in the corresponding score-function
update. These are uncentered score-function coefficients;
adding a state-only baseline does not change the expected update.

In contrast, sampled-token OPD uses the local signal defined in
the preliminaries,
\[
A_t
=
\log\frac{\pi^E(y_t\mid s_t)}
         {\pi_\theta(y_t\mid s_t)},
\]
and therefore drops the successor value
$V_\theta(f(s_t,y_t))$. This omission arises from local credit
assignment rather than single-token sampling: differentiating
the full-vocabulary local reverse KL at detached student-generated
prefixes also omits the effect of the current action on future
visited prefixes.

Importantly, the omitted term is zero when $y_t=e$, since
termination transitions directly to the absorbing state.
Nevertheless, it contributes to the stopping-logit signal
through softmax coupling with continuation tokens.

\paragraph{Effect on the stopping logit.}
The difference between the two updates is particularly simple for the
stopping action. Let
\[
p_\theta(s)=\pi_\theta(e\mid s),
\qquad
C_\theta(s)
=
\mathbb{E}_{y\sim\pi_\theta(\cdot\mid s)}
\left[
V_\theta(f(s,y))
\right].
\]
Since $V_\theta(f(s,e))=0$, $C_\theta(s)$ contains only the successor
KL associated with continuation actions. We can further have the following proposition:

\begin{proposition}[Missing stopping pressure in local OPD]
\label{prop:transient_eos_bias}
Let $\Delta_e^{*}(s)$ and $\Delta_e^{\mathrm{OPD}}(s)$ denote the expected
ascent updates to the stopping logit $z_\theta(s,e)$ under the
trajectory-consistent and local OPD coefficients, respectively. At any fixed prefix $s$,
\[
\Delta_e^{*}(s)
-
\Delta_e^{\mathrm{OPD}}(s)
=
p_\theta(s)C_\theta(s)
\ge 0.
\label{eq:transient_eos_gap}
\]
Thus, local sampled token OPD always applies weakly less upward pressure
to the stopping logit than the trajectory-consistent reverse KL update.
\end{proposition}

\begin{proof}
For any action $b$, softmax coupling gives
\[
\frac{\partial \log \pi_\theta(y\mid s)}
{\partial z_\theta(s,b)}
=
\mathbf{1}\{y=b\}-\pi_\theta(b\mid s).
\]
Since the trajectory-consistent coefficient differs from local OPD by
$-V_\theta(f(s,y))$, their expected update difference on $z_\theta(s,b)$ is
\[
\pi_\theta(b\mid s)
\left[
C_\theta(s)-V_\theta(f(s,b))
\right],
\]
where
\[
C_\theta(s)
=
\mathbb{E}_{y\sim\pi_\theta(\cdot\mid s)}
[V_\theta(f(s,y))].
\]
For the stopping action $e$, $f(s,e)=\dagger$ and
$V_\theta(\dagger)=0$. Therefore,
\[
\Delta_e^{*}(s)-\Delta_e^{\mathrm{OPD}}(s)
=
p_\theta(s)C_\theta(s)\ge 0.
\]
\end{proof}

\paragraph{Interpretation.}
\Cref{prop:transient_eos_bias} describes a relative bias rather than an
absolute suppression of stopping. Local OPD may still increase the stopping
logit when the immediate teacher signal favors termination. The proposition
instead states that, compared with the trajectory consistent reverse KL
update, local OPD always misses the nonnegative stopping signal
\[
p_\theta(s)C_\theta(s).
\]

The source of this gap is the continuation actions rather than the stopping
action itself. After stopping, the successor KL is zero. For a continuation
action, however, the trajectory level objective additionally subtracts the
future mismatch $V_\theta(f(s,a))$. Through softmax coupling, penalizing these
continuation actions also increases the relative preference for stopping.
Local sampled token OPD omits these successor penalties and therefore
overweights continuation relative to the trajectory consistent update.

This difference is naturally largest when the student still differs
substantially from the teacher over future continuations. From
\eqref{eq:future_kl_bellman},
\[
V_\theta(s)=k_\theta(s)+C_\theta(s),
\qquad
0\leq C_\theta(s)\leq V_\theta(s),
\]
so that
\[
V_\theta(s)\rightarrow 0
\quad\Longrightarrow\quad
p_\theta(s)C_\theta(s)\rightarrow 0.
\]
Hence, as distillation reduces the future student teacher mismatch, the local
and trajectory consistent updates become increasingly similar.

This provides a possible explanation for the transient length drift observed
after correcting the termination token mismatch. As shown in
\Cref{fig:model_eos_cmp}, stopping probabilities are initially low and recover
later in training, while response lengths first increase and then decrease.
The effect is especially pronounced for Llama 3.2 and Gemma 3, where the low
termination regime persists for longer. We therefore view the omitted future
mismatch as a plausible contributor to these transient dynamics, rather than
as a guarantee that local OPD must increase response length.

This mechanism is distinct from the termination token mismatch in
\Cref{sec:causes}, which explains the persistent budget saturation observed
in vanilla Qwen3 OPD.
\section{Additional Experimental Results}
\subsection{Termination Mismatch and Residual Length Dynamics Across Model Families}
\begin{figure}[t]
    \centering
    \includegraphics[width=\linewidth]{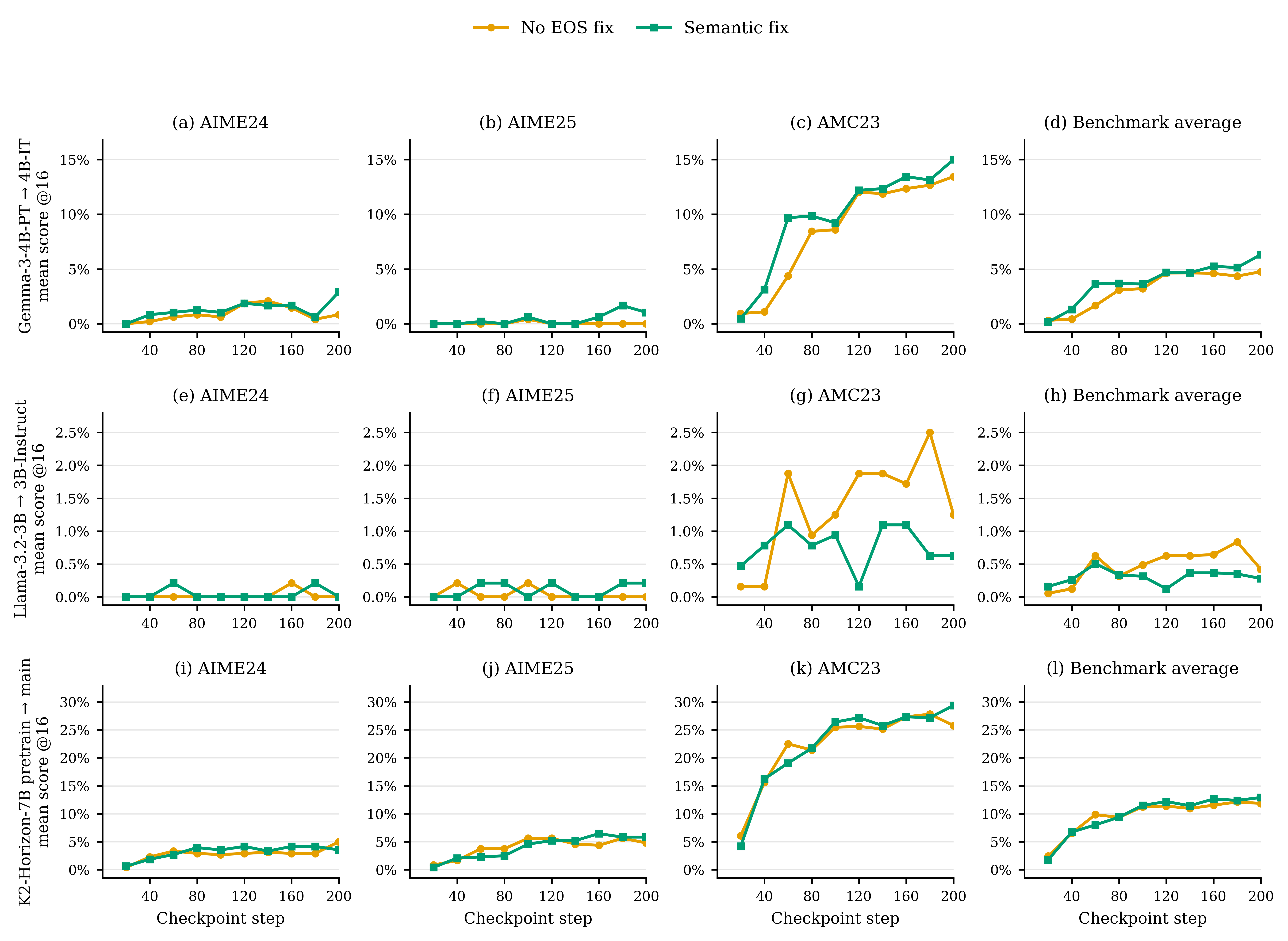}
    \caption{Evaluation results with and without the semantic EOS fix
    under the TTRL evaluation template for Gemma-3-4B, Llama-3.2-3B,
    and K2-Horizon-7B pretrain\_final distillation.
    Columns show mean scores over 16 samples per problem on AIME24,
    AIME25, and AMC23, followed by their macro average.}
    \label{fig:models_result}
\end{figure}

We report additional evaluation results for distillation with
Gemma-3-4B, Llama-3.2-3B, and K2-Horizon-7B pretrain\_final in
\Cref{fig:models_result}.
Under the TTRL evaluation template, the semantic EOS fix yields modest, model-dependent changes in single-turn mathematical reasoning
performance.  For Llama, both variants remain near zero throughout training Together, these results indicate that improvements in termination behavior do not necessarily translate into consistent gains in mathematical reasoning performance.

\begin{figure}[t]
    \centering
    \includegraphics[width=\linewidth]{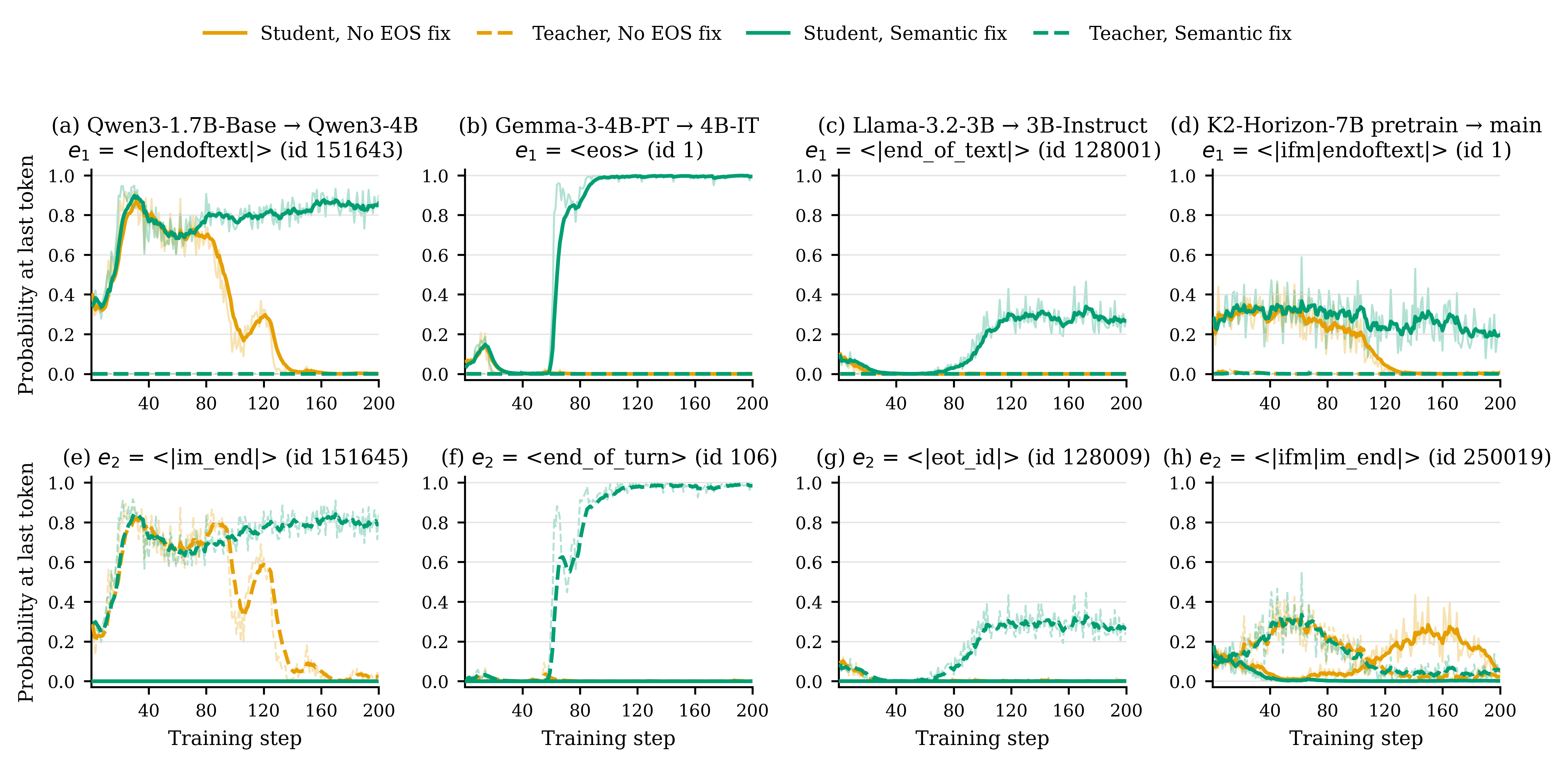}
    \caption{
    \textbf{Evolution of individual termination token probabilities.}
    We track the probabilities assigned by the student and teacher to each termination token at the final position of student rollouts for Qwen3 1.7B, Gemma 3 4B, Llama 3.2 3B, and K2 Horizon 7B.
    Each column is one student--teacher pair; the top row shows the document-level terminator $e_1$ and the bottom row the turn-level
    terminator $e_2$. Llama 3.2 defines a third terminator, \texttt{<|eom\_id|>}, which we omit: its probability stays below $0.02$ throughout training for both the student and the teacher under both settings.
    }
    \label{fig:terminal_token_grid}
\end{figure}

To better understand the cross family length behavior, we examine how the
student and teacher distribute probability across their individual
termination tokens in \Cref{fig:terminal_token_grid}. For Gemma 3 and
Llama 3.2, vanilla OPD exhibits a clear termination preference mismatch.
In Gemma 3, the student primarily uses \texttt{<eos>}, whereas the teacher
strongly prefers \texttt{<end\_of\_turn>}. Similarly, in Llama 3.2, the
student primarily places termination mass on \texttt{<|end\_of\_text|>},
while the teacher prefers \texttt{<|eot\_id|>}. Consequently, vanilla OPD
suppresses the termination token naturally used by the student without
reliably transferring the teacher's preferred surface form. The semantic
correction avoids this competition by matching their total probability of
termination while allowing the two models to retain different surface token
preferences.

K2 Horizon exhibits a different two-stage behavior. Early in training, the teacher and student prefer different termination tokens, but unlike Qwen3, Llama, and Gemma, the student's probability on the teacher-preferred token is not vanishingly small. The token can therefore still be sampled with non-negligible probability, allowing sampled-token OPD to receive direct teacher supervision on it and gradually transfer termination mass toward the
teacher-preferred surface form.

Later in training, however, the probabilities of both termination tokens decrease again. This second-stage decline cannot be explained by the surface
EOS mismatch alone, since the teacher-preferred token has already been learned. We hypothesize that it is instead related to the local nature of sampled-token OPD: when the teacher assigns low termination probability at the current student-generated prefix, the local update can continue to favor continuation while ignoring the future student-teacher mismatch induced by doing so. We
analyze this possible mechanism in \Cref{sec:length_drift}.

This suggests that termination token identity alone does not explain all
of the observed length dynamics. At prefixes where the teacher itself
assigns little total probability to termination, sampled token OPD can still
favor continuation even after the semantic mismatch is corrected. We study
this additional effect in \Cref{sec:length_drift}, where we analyze how the
local OPD update differs from the trajectory level reverse KL update by
ignoring future student teacher mismatch.

\subsection{Semantic EOS Correction at Later Training Stages}
\label{app:k2_later_stage_semantic}

The K2-Horizon experiments in \Cref{sec:k2_stages} show that the midtraining
and SFT checkpoints already assign substantial probability to the
teacher-preferred termination token. Unlike the pretrained initialization,
vanilla OPD from these later-stage checkpoints does not exhibit the severe
termination collapse observed in the base-to-post-trained settings studied
in the main text. In these cases, semantic EOS correction is therefore not
needed to recover the teacher's preferred termination form.

We additionally test whether applying the correction when the termination
mismatch is already substantially reduced perturbs either the training
dynamics or downstream performance. This serves as a non-interference check
for using semantic EOS aggregation as a general termination-handling rule.

\begin{figure}[t]
    \centering
    \includegraphics[width=\linewidth]
    {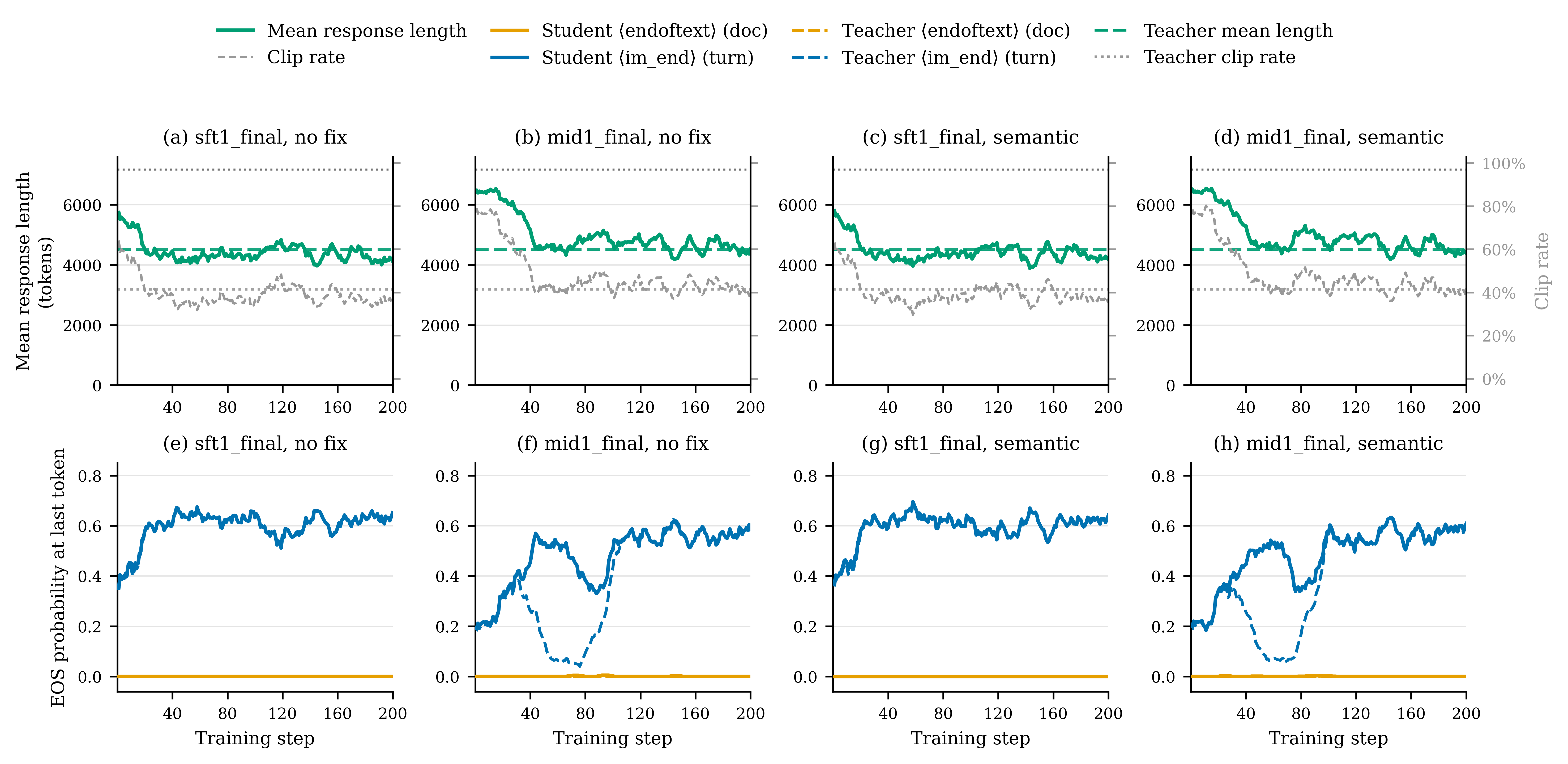}
    \caption{
    \textbf{Semantic EOS correction at later K2-Horizon training stages.}
    We distill the final post-trained K2-Horizon model into students
    initialized from the SFT and midtraining checkpoints, comparing vanilla
    OPD with semantic EOS correction. The top row shows mean response length
    and clipping rate together with the corresponding teacher reference
    levels, while the bottom row shows the student and teacher probabilities
    assigned to \texttt{<|ifm|endoftext|>} and
    \texttt{<|ifm|im\_end|>} at the final observed rollout position.
    In both initialization settings, the corrected and uncorrected runs
    exhibit similar length, clipping, and termination-probability dynamics.
    In particular, semantic aggregation does not introduce the termination
    failure that it is designed to prevent in more severely mismatched
    settings.
    }
    \label{fig:k2_later_stage_semantic_dynamics}
\end{figure}

As shown in \Cref{fig:k2_later_stage_semantic_dynamics}, the semantic
correction has little effect on the later-stage K2 training trajectories.
For both the SFT and midtraining initializations, the student already places
substantial mass on \texttt{<|ifm|im\_end|>}, and its termination behavior
remains stable under vanilla OPD. Applying semantic aggregation leaves the
response-length and clipping trajectories broadly unchanged and produces
similar evolution of the termination probabilities. Thus, when the student
and teacher are already sufficiently aligned at the level of termination,
the correction is largely inactive rather than forcing a different surface
termination convention.

\begin{figure}[t]
    \centering
    \includegraphics[width=\linewidth]
    {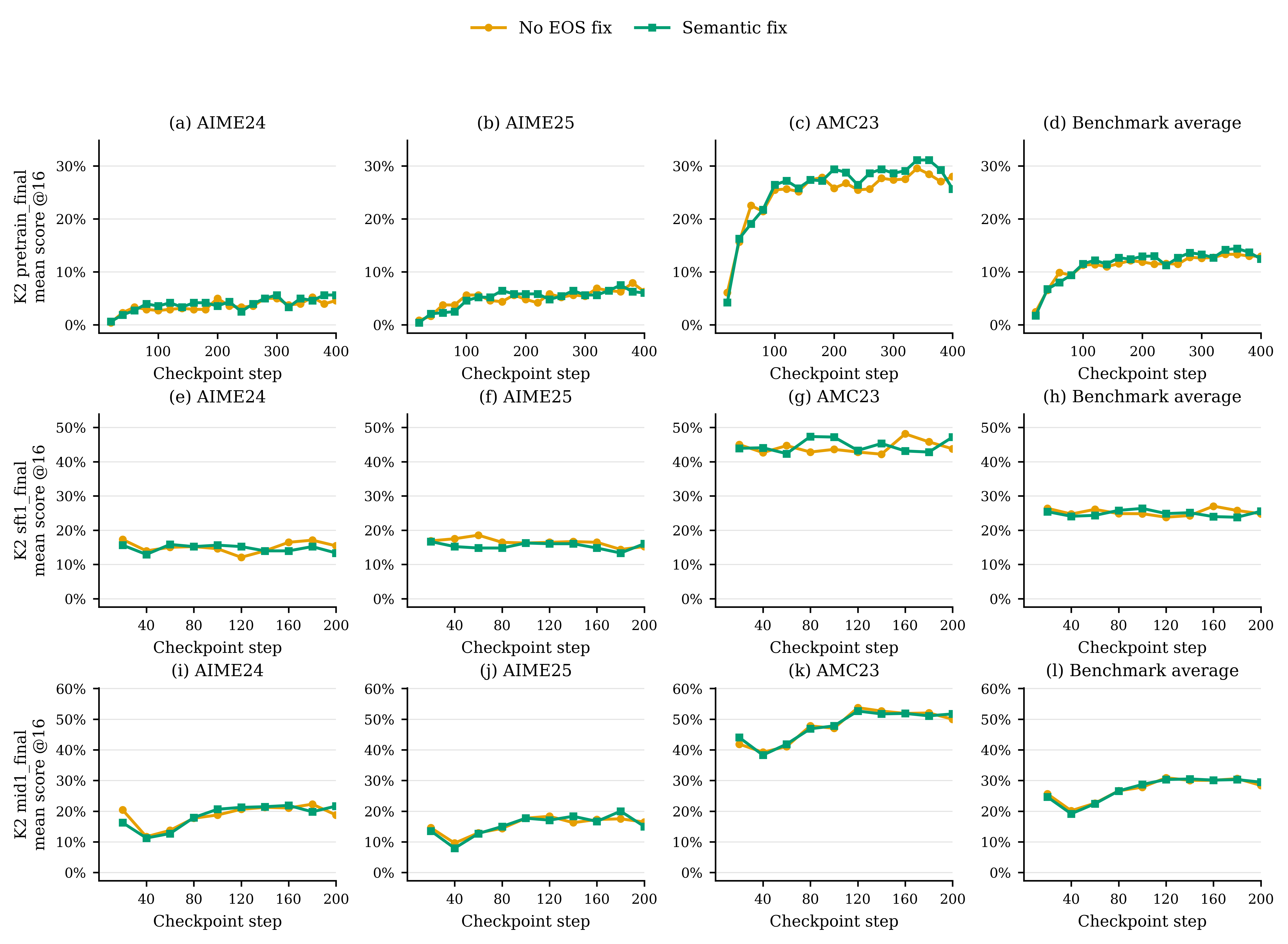}
    \caption{
    \textbf{Downstream evaluation with and without semantic EOS correction
    across K2-Horizon training stages.}
    We report Avg@16 on AIME24, AIME25, and AMC23, together with their
    unweighted benchmark average, for Pretrain-to-Final, SFT-to-Final, and
    Midtrain-to-Final distillation. The Pretrain-to-Final runs are evaluated
    through 400 training steps, while the SFT-to-Final and Midtrain-to-Final
    runs are evaluated through 200 steps. Across initialization stages and
    benchmarks, semantic EOS correction yields performance similar to vanilla
    OPD, with no consistent degradation in downstream accuracy.
    }
    \label{fig:k2_stage_semantic_eval}
\end{figure}

The downstream results in \Cref{fig:k2_stage_semantic_eval} show the same
pattern. For the SFT and midtraining initializations, the semantic and
uncorrected curves remain close across AIME24, AIME25, AMC23, and their
average, with small checkpoint-level differences in both directions. We do
not observe a systematic accuracy penalty from aggregating the
termination-equivalent tokens. The Pretrain-to-Final evaluation, included
for completeness, similarly shows broadly comparable performance between
the two variants despite their different treatment of termination.

Together, these results provide a useful control for the semantic EOS
correction. Its main benefit arises in settings where surface termination
preferences are strongly mismatched and sampled-token OPD cannot reliably
transfer probability to the teacher-preferred token. When the student
already provides substantial support for that termination form, as in the
later K2-Horizon checkpoints, semantic aggregation leaves both the training
dynamics and downstream performance largely unchanged in our experiments.
This supports its use as a general correction without requiring that every
teacher--student pair exhibit severe termination mismatch.
\section{Influence of Training and Evaluation Templates}
\label{sec:template_influence}

\begin{figure}[t]
    \centering
    \includegraphics[width=\linewidth]{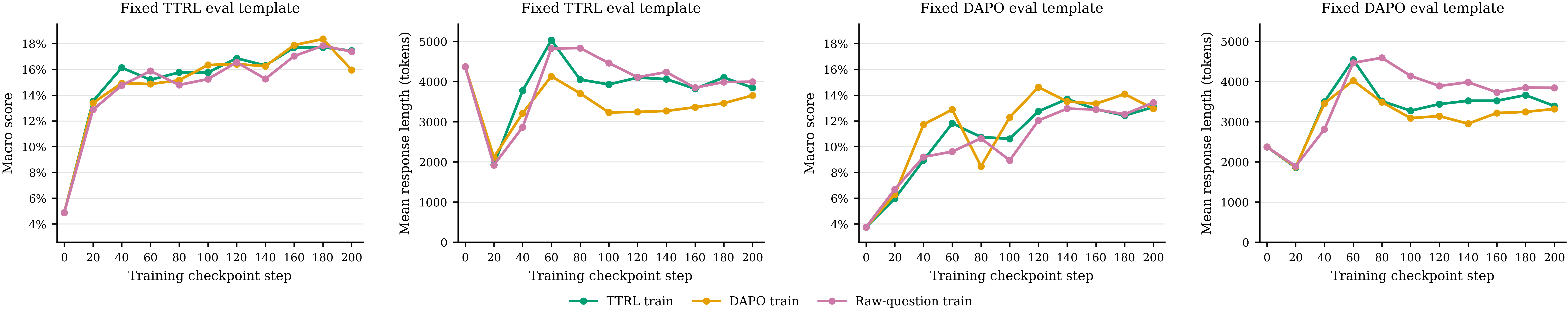}
    \caption{
    \textbf{Effect of Training and Evaluation Prompt Templates on Model Behavior.} Influence of training and evaluation prompt templates on downstream performance and response length. Models are trained using the TTRL, DAPO, or raw-question template and evaluated with either the TTRL template (left) or the DAPO template (right). The semantic-EOS-class fix is applied in all experiments.
    }
    \label{fig:semantic_eval}
\end{figure}

Recall that \Cref{fig:baseline_eval} suggests that changing the evaluation
format relative to the training prompt can substantially alter the observed
response length. We therefore study this interaction more systematically.
Since \Cref{fig:fix_cmp} shows that the EOS correction methods exhibit
broadly similar behavior, we use the semantic EOS class as the default
correction in this study.

Following the setup described above, we train Qwen3-1.7B-Base with Qwen3-4B
as the teacher using each of the three training prompt templates, and
evaluate every checkpoint under both the TTRL and DAPO formats. All other
experimental settings are kept fixed.

The results are shown in \Cref{fig:semantic_eval}. We make three main
observations.

\paragraph{Obs 1: The evaluation format strongly affects measured performance.}
Evaluation under the TTRL format generally yields higher measured accuracy
than evaluation under the DAPO format across all three training settings.
Matching the evaluation format to the training prompt does not provide a
consistent advantage. For example, models trained with DAPO still perform
better under TTRL evaluation than under DAPO evaluation. Since the two
evaluation formats differ in both prompting and grading, we interpret this
as an effect of the overall evaluation protocol rather than of the prompt
alone.

\paragraph{Obs 2: The training prompt induces a persistent response length bias.}
Models trained with the DAPO template generally produce shorter responses
than those trained with the TTRL or raw question templates, and this tendency
persists under both evaluation formats. This suggests that the training
prompt induces a persistent change in generation behavior that transfers
across evaluation prompts.

\paragraph{Obs 3: Response length and measured performance are partially decoupled.}
The shorter responses induced by DAPO training do not correspond to a
consistent reduction in performance. Under TTRL evaluation, models trained
with DAPO achieve performance comparable to the other training settings
while producing substantially shorter responses. Thus, changes in response
length and measured task performance need not move together.

Overall, these results suggest that the training prompt can persistently
shape generation behavior, while the evaluation format can substantially
affect measured performance. A more systematic characterization of this
interaction remains an important direction for ongoing work.

\section{Rollout Examples from \Cref{fig:demo}}
\label{app:trajectories}

This section provides the rollouts corresponding to the three examples in
\Cref{fig:demo}. Since each response reaches the full generation budget, we
show the reasoning up to the first correct answer together with representative
parts of the subsequent repetitive continuation, omitting repeated middle
segments for readability.

\subsection{Example 1: Answer Repetition}
\label{app:traj_answer_repetition}

This example is AMC23 problem 10. The correct answer is first produced after
1,094 tokens, followed by 7,098 redundant tokens, accounting for 86.6\%
of the response.

\begin{rollout}{Prompt}
\begin{Verbatim}
Rows 1, 2, 3, 4, and 5 of a triangular array of integers are shown below.
1
1 1
1 3 1
1 5 5 1
1 7 11 7 1
Each row after the first row is formed by placing a 1 at each end of the row, and each interior entry is 1 greater than the sum of the two numbers diagonally above it in the previous row. What is the units digits of the sum of the 2023 numbers in the 2023rd row? Please reason step by step, and put your final answer within \boxed{}.
\end{Verbatim}
\end{rollout}

\begin{rollout}{Reasoning up to the first correct answer}
\begin{Verbatim}
[...]
-> units digit = 8
- $2^8 = 256$ -> units digit = 6

We can see that the **units digit of powers of 2** cycles every **4** numbers:

- 2, 4, 8, 6, 2, 4, 8, 6, ...

So the **pattern repeats every 4 powers**.

Let's find the **position of 2024 in this cycle**:

- 2024 / 4 = 506 with a remainder of 0

So, the units digit of $2^{2024}$ is the same as the units digit of $2^4$, which is **6**

Then, the **units digit of $2^{2024} - 1$ is 6 - 1 = 5**

---

### Final Answer:

\boxed{5}
\end{Verbatim}
\end{rollout}

\begin{rollout}{Repetitive continuation}
\begin{Verbatim}
---

### Final Answer:

\boxed{5}

The **units digit of the sum of the numbers in the 2023rd row** is **5**.

---

### Final Answer:

\boxed{5}

---

### Conclusion:

- The **units digit of the sum of the numbers in the 2023rd row** is **5**.

---

### Final Answer:

\boxed{5}

---

### Final Answer:

\boxed{5}

---

### Final Answer:

\boxed{5}

---

### Final Answer:

\boxed{5}

---

### Final Answer:

\boxed{5}

     [ the '### Final Answer' block above repeats 704 times in total ]

---

### Final Answer:

\boxed{5}

---

### Final Answer:

\boxed{5}

---

### Final Answer:

\boxed{5}

---

### Final

[ truncated at the 8192-token budget; no EOS token was emitted ]
\end{Verbatim}
\end{rollout}

\subsection{Example 2: Self Correction Loop}
\label{app:traj_self_correction_loop}

This example is AIME24 problem 7. The correct answer is first produced after 759 tokens, followed by 7,433 redundant tokens, accounting for 90.7\%
of the response.

\begin{rollout}{Prompt}
\begin{Verbatim}
There exist real numbers $x$ and $y$, both greater than 1, such that $\log_x\left(y^x\right)=\log_y\left(x^{4y}\right)=10$. Find $xy$. Please reason step by step, and put your final answer within \boxed{}.
\end{Verbatim}
\end{rollout}

\begin{rollout}{Reasoning up to the first correct answer}
\begin{Verbatim}
[...]
ifferent bases:

We know that:

$$
\log_x(y) = \frac{\log_y(y)}{\log_y(x)} = \frac{1}{\log_y(x)} = \frac{1}{b}
$$

So:

$$
a = \frac{1}{b} \quad \Rightarrow \quad \frac{10}{x} = \frac{1}{\frac{5}{2y}} = \frac{2y}{5}
$$

So:

$$
\frac{10}{x} = \frac{2y}{5} \quad \Rightarrow \quad 10 = \frac{2y}{5} \cdot x \quad \Rightarrow \quad 10 = \frac{2xy}{5}
$$

Multiply both sides by 5:

$$
50 = 2xy \quad \Rightarrow \quad xy = \frac{50}{2} = 25
$$

---

### Final Answer:

$$
\boxed{25}
\end{Verbatim}
\end{rollout}

\begin{rollout}{Repetitive continuation}
\begin{Verbatim}
$$

---

### Conclusion:

We found that:

$$
xy = \boxed{25}
$$

**Answer:** $\boxed{25}$

---

### Verification:

We can verify the solution:

Let's assume $xy = 25$, and solve for $x$ and $y$.

From $xy = 25$, we can choose $x = 5$ and $y = 5$, so $xy = 25$.

Now, check the two logarithmic equations:

1. $\log_x(y^x) = \log_5(5^5) = \log_5(3125) = 5$ (not 10, so this is incorrect).

We need to choose values such t

     [ the 'Let's try' block above repeats 371 times in total ]

et's try $x = 2$, $y = 12.5$.

Let's try $x = 2$, $y = 12.5$.

Let's try $x = 2$, $y = 12.5$.

Let's try $x = 2$, $y = 1

[ truncated at the 8192-token budget; no EOS token was emitted ]
\end{Verbatim}
\end{rollout}

\subsection{Example 3: Token Repetition}
\label{app:traj_token_repetition}

This example is AMC23 problem 0. The correct answer is first produced after
462 tokens, followed by 7,730 redundant tokens, accounting for 94.4\%
of the response. Two non ASCII symbols in the original rollout are
rendered as \texttt{[+]} and \texttt{[!]} below for typesetting
compatibility.

\begin{rollout}{Prompt}
\begin{Verbatim}
Cities $A$ and $B$ are $45$ miles apart. Alicia lives in $A$ and Beth lives in $B$. Alicia bikes towards $B$ at 18 miles per hour. Leaving at the same time, Beth bikes toward $A$ at 12 miles per hour. How many miles from City $A$ will they be when they meet? Please reason step by step, and put your final answer within \boxed{}.
\end{Verbatim}
\end{rollout}

\begin{rollout}{Reasoning up to the first correct answer}
\begin{Verbatim}
[...]
ances they travel** should equal the **total distance between them**, which is 45 miles.

So:

$$
18t + 12t = 45
$$

---

### Step 3: Solve for $t$

$$
30t = 45
$$

$$
t = \frac{45}{30} = 1.5 \text{ hours}
$$

So, they meet after **1.5 hours**.

---

### Step 4: Find the distance from City $A$

Alicia is biking towards $B$ at 18 mph, so in 1.5 hours, she will have traveled:

$$
\text{Distance from } A = 18 \cdot 1.5 = 27 \text{ miles}
$$

---

### Final Answer:

$$
\boxed{27}
\end{Verbatim}
\end{rollout}

\begin{rollout}{Repetitive continuation}
\begin{Verbatim}
$$

They will be **27 miles from City $A$** when they meet. [+]

--- 

### Final Answer:

$$
\boxed{27}
$$

They meet **27 miles from City $A$**. [!] [!] [!] [+] [+] [+] [+] [+] [+] [+] [+] [+] [+] [+] [+] [+] [+] [+] [+] [+] [+] [+] [+] [+] [+] [+] [+] [+] [+] [+] [+] [+] [+] [+] [+] [+] [+] [+] [+] [+] [+] [+] [+] [+] [+] [+] [+] [+] [+] [+] [+] [+] [+] [+] [+] [+] [+] [+]

     [ the check-mark token above repeats 3836 times in total ]

 [+] [+] [+] [+] [+] [+] [+] [+] [+] [+] [+] [+] [+] [+] [+] [+] [+] [+] [+] [+] [+] [+] [+] [+] [+] [+] [+] [+] [+] [+]

[ truncated at the 8192-token budget; no EOS token was emitted ]
\end{Verbatim}
\end{rollout}

\end{document}